\documentclass{article}

\usepackage{microtype}
\usepackage{graphicx}
\usepackage{subcaption}
\usepackage{booktabs} 

\usepackage{hyperref}

\usepackage[accepted]{icml2026}

\usepackage{amsmath}
\usepackage{amssymb}
\usepackage{mathtools}
\usepackage{amsthm}
\usepackage{subcaption}
\usepackage{slashed}
\usepackage[table]{xcolor} 

\floatname{algorithm}{Algorithm}         
\usepackage{algorithm}
\usepackage{algorithmic}
\usepackage{amsmath}            
\renewcommand{\algorithmicrequire}{\textbf{Input:}}  
\renewcommand{\algorithmicensure}{\textbf{Output:}}  
\usepackage[capitalize,noabbrev]{cleveref}

\theoremstyle{plain}
\newtheorem{theorem}{Theorem}[section]

\newtheorem{lemma}{Lemma}[section] 

\theoremstyle{definition}

\newtheorem{remark}{Remark}[section]

\newif\ifshowchanges
\showchangesfalse  
\ifshowchanges
  \newcommand{\rev}[1]{\textcolor{blue}{#1}}
  \newcommand{\revtab}[1]{\begingroup\color{blue}#1\endgroup}
\else
  \newcommand{\rev}[1]{#1}
  \newcommand{\revtab}[1]{#1}
\fi
\usepackage[textsize=tiny]{todonotes}

\icmltitlerunning{Q-Guided Alignment for Return-Conditioned Supervised Learning}

\usepackage{enumitem}
\begin{document}
\newcommand{\methodname}{\textsc{Q-align DT}}
\twocolumn[
  \icmltitle{Return-to-Go Is More Than a Number: Q-Guided Alignment for Return-Conditioned Supervised Learning}



  \icmlsetsymbol{equal}{*}

  \begin{icmlauthorlist}
    \icmlauthor{Yuxiao Yang}{aaa}
    \icmlauthor{Weitong Zhang}{aaa}
  \end{icmlauthorlist}

  \icmlaffiliation{aaa}{University of North Carolina at Chapel Hill}

  \icmlcorrespondingauthor{Yuxiao Yang}{yxyang@unc.edu}
  \icmlcorrespondingauthor{Weitong Zhang}{weitongz@unc.edu}

  \icmlkeywords{Machine Learning, ICML}

  \vskip 0.3in
]



\printAffiliationsAndNotice{}  

\begin{abstract}
Conditioned Sequence Models (CSMs) learn policies by treating return-to-go (RTG) as a control signal.
However, existing CSMs often treat the RTGs as simple numerical inputs rather than aligning them with the performance of their policies.
In this paper, we propose \methodname, a framework that enforces this alignment by ensuring the $Q$-value of the output policy is consistent with the input RTG. 
By leveraging a $Q$ function to provide dense guidance to CSMs and further fine-tuning it using an \textit{RTG-perturbation} technique with the CSM, our method ensures that higher RTGs are consistently mapped to trajectories with higher expected returns.
Theoretically, we show that \methodname{} can efficiently learn the desired policy and output a near-optimal one when the RTG is sufficiently high.
Empirically, we demonstrate through extensive experiments that \methodname{} achieves superior controllability and performance across the D4RL benchmark. Remarkably, our model effectively learns a structured family of policies that maintains precise alignment and generalizes to tasks like velocity-tracking where prior methods fail. 
\end{abstract}

\section{Introduction}
Offline reinforcement learning (Offline RL) aims to learn an effective and robust policy that can be deployed without interacting with the environment, relying solely on pre-collected datasets \citep{Offline-RL}.
Recently, transformer architectures \citep{attention}, which have shown remarkable success in natural language processing \citep{Bert,gpt3,roberta} and computer vision \citep{attention-cv,mae,vit,swin}, have also been adopted in RL due to their powerful sequence modeling capabilities \citep{AD,dpt,AMAGO}.
Among these advances, Conditional Sequence Models (CSMs) \citep{DT,tt} provide a new perspective by treating policy learning as a supervised sequence modeling problem conditioned on the desired performance. 
In particular, Decision Transformer (DT) introduces a Return-to-Go (RTG) token which enables the model to generate trajectories that achieve expected returns rather than imitating the behavioral policy.


However, how the actual return obtained by a CSM \textit{aligns} with the target input RTG is a fundamental yet often 
overlooked property. Precise alignment enables a single model to represent a diverse family of policies and is essential for controllable robot behaviors with varying velocities~\citep{RADT}. 
Therefore, we would like to ask: 
\begin{center}
    \emph{How well can CSMs actually align with target RTGs?}
\end{center}
Unfortunately, recent studies~\citep{DC, RADT} report that existing CSMs often exhibit significant insensitivity to RTG and fail to achieve proper alignment. As empirically demonstrated in~\cref{fig:cmp}, we hypothesize that this failure stems from a lack of structural awareness: a robust CSM should capture the structure between different RTG targets and their corresponding behaviors instead of merely treating the RTGs as simple tokens.
Specifically, a higher RTG should consistently correspond to trajectories with higher expected cumulative returns, ensuring this \textit{partial order} between target RTGs and realized performance. 

In offline RL, enforcing such a partial order is challenging since it can be infeasible to construct sufficient trajectories within the fixed dataset following the partial order. 
To address this challenge, in this paper, we propose \methodname, which introduces an auxiliary $Q$-function to provide dense guidance. 
With a novel \emph{RTG-to-behavior} objective alongside an \emph{RTG-perturbation} technique, our method encourages the model to output actions that precisely reflect the relative differences in desired RTGs. We further integrate this RTG perturbation into the $Q$-function updates during co-training, ensuring that the critic and the policy co-evolve toward a consistent, reward-sensitive behavior.

Building on our algorithm, we provide a theoretical analysis of the alignment properties of CSMs and conduct extensive experiments to evaluate our trained models. We find that \methodname~learns a family of RTG-conditioned policies that actively respond to target shifts, rather than merely relying on static reward associations found in the training data. Moreover, we demonstrate that our model can be adapted to distinct tasks (e.g., \texttt{HalfCheetah-Vel}) while maintaining competitive performance and alignment, indicating its potential for generalization across diverse behaviors.

Overall, our contributions are threefold:

\begin{itemize}[leftmargin=*, topsep=0pt,itemsep=0pt,partopsep=0pt,parsep=0pt]
\item   We propose \methodname{}, which introduces a new RTG-to-behavior alignment objective to enforce consistency between the input RTG and policy behavior, significantly improving RTG-conditioned alignment.  
It further employs a co-training framework with RTG perturbations that provide a high-quality action space for $Q$-function learning, enabling bidirectional improvements.

\item Theoretically, we show that \methodname{} improves alignment by restricting the policy class, and that the alignment objective ensures, under high-RTG conditioning, equivalence to maximizing the $Q$-function. 
Combined with RTG perturbations, this leads to convergence to a near-optimal in-distribution policy under mild assumptions.

\item Extensive experiments show that \methodname{} consistently achieves competitive performance across a wide range of offline RL tasks while significantly improving RTG-conditioned alignment. Remarkably and somewhat surprisingly, we report that \methodname{} generalizes effectively to the challenging \texttt{HalfCheetah-Vel} task and attains competitive performance by only controlling the RTG signal with zero-shot transfer.

\end{itemize}
\textbf{Code.}
Our code is available at \url{https://github.com/yangyuxiao-sjtu/Q-Align-DT}.

\begin{figure*}[ht]
\centering
\begin{subfigure}[b]{0.32\textwidth}
\centering
\includegraphics[width=\textwidth]{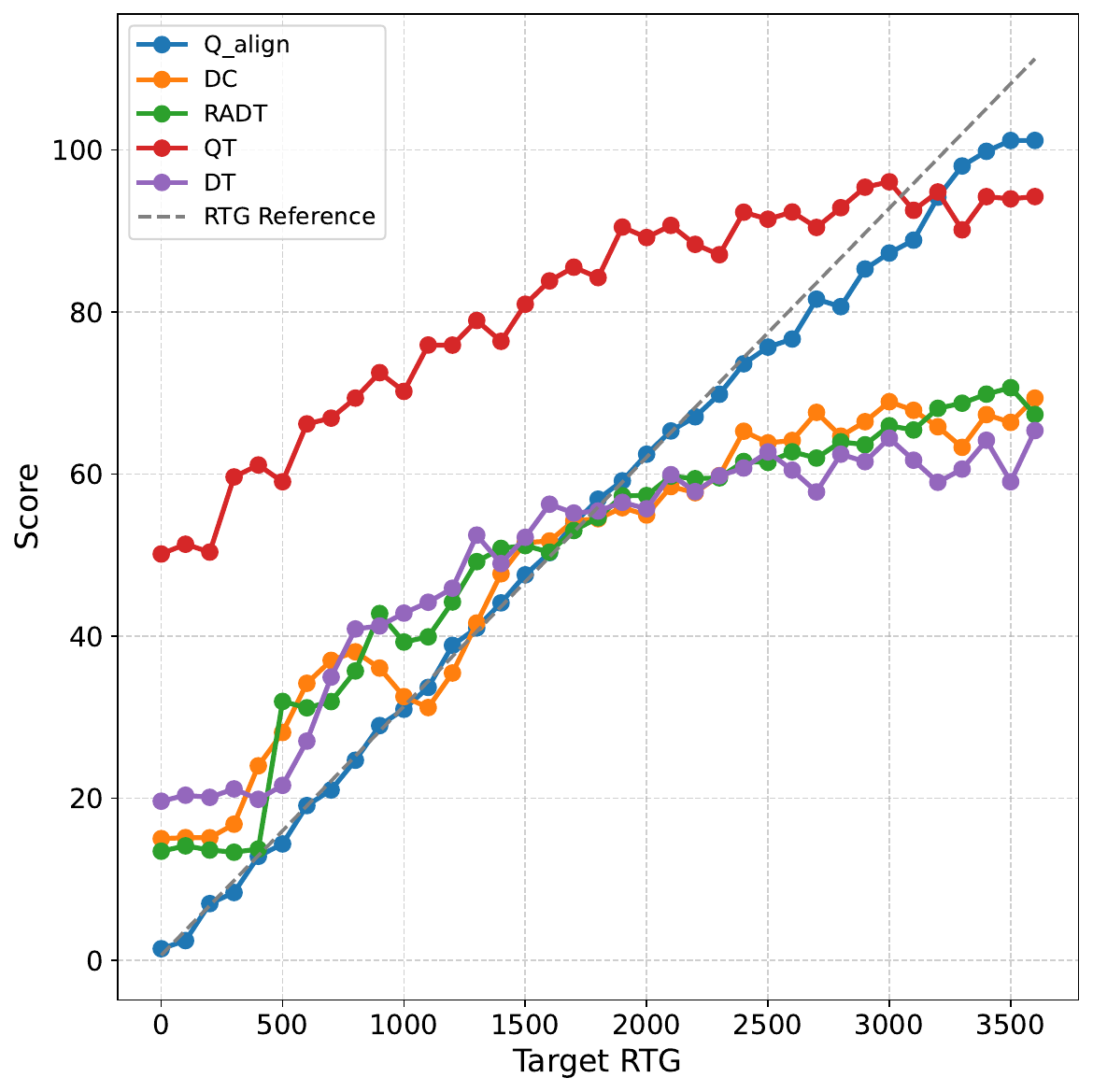}
\vspace{-1em}
\caption{hopper-medium}
\label{fig:sub1}
\end{subfigure}
\hfill
\begin{subfigure}[b]{0.32\textwidth}
\centering
\includegraphics[width=\textwidth]{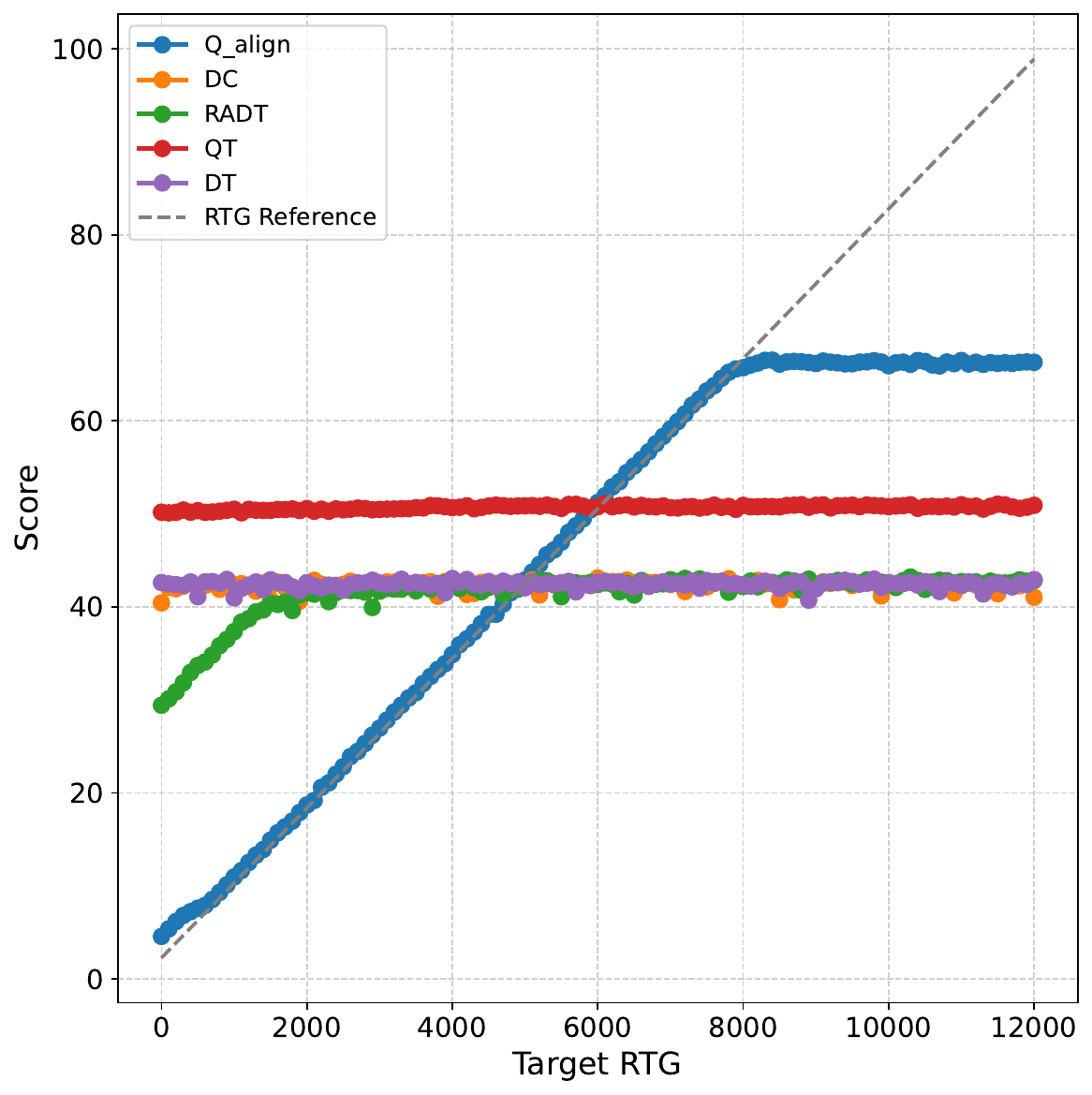}
\vspace{-1em}
\caption{halfcheetah-medium}
\label{fig:sub2}
\end{subfigure}
\hfill
\begin{subfigure}[b]{0.32\textwidth}
\centering
\includegraphics[width=\textwidth]{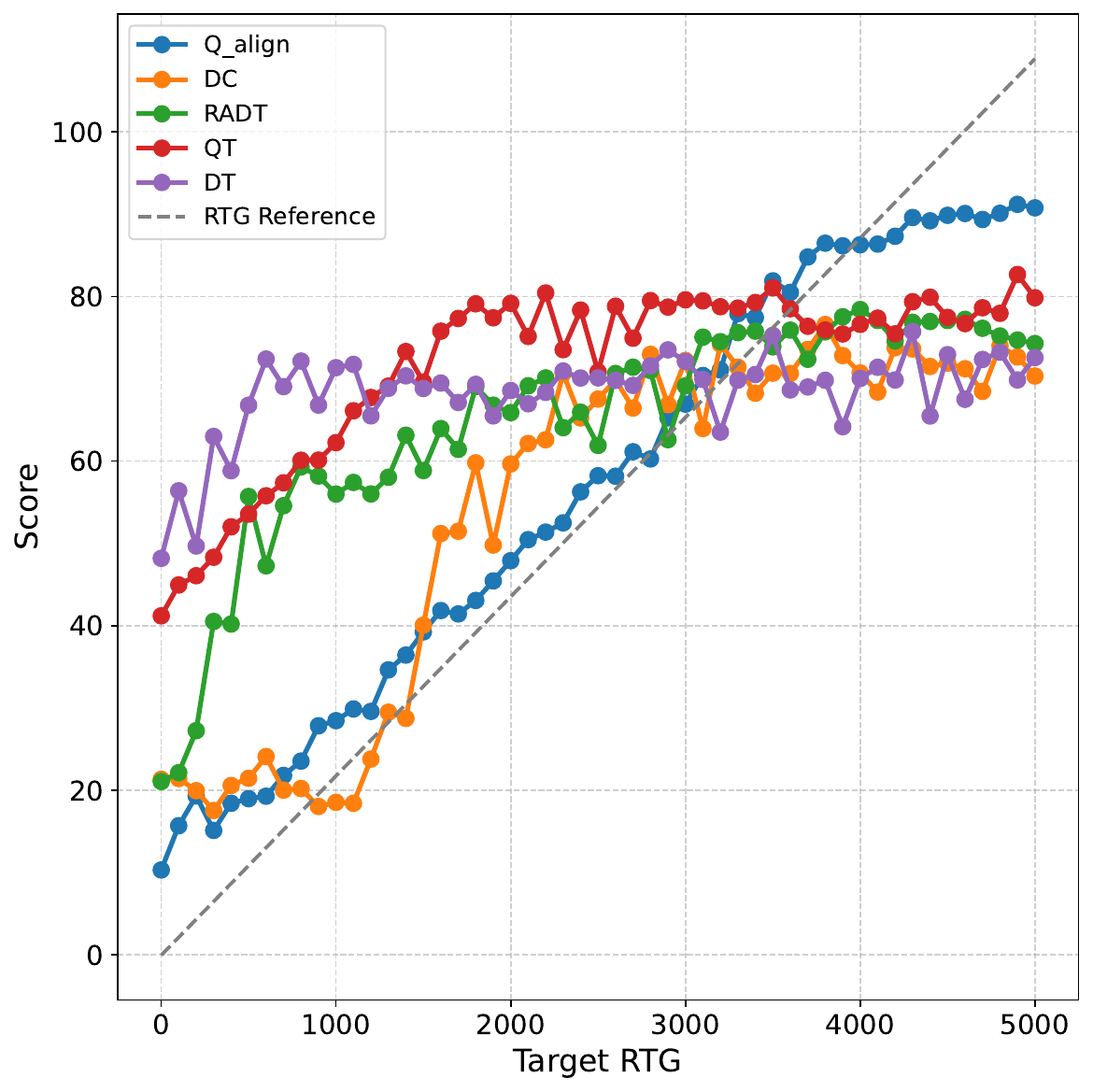}
\vspace{-1em}
\caption{walker2d-medium}
\label{fig:sub3}
\end{subfigure}
\vspace{-0.5em}
\caption{Performance of \methodname~(Q\_align) and other baseline models on D4RL tasks. Target RTGs are set with an interval of 100 targeting for cumulative rewards. We sample 30 trajectories for each target RTG and report the mean performance for each method.}

    \label{fig:cmp}
\end{figure*}

\section{Related Works}
\subsection{Offline Reinforcement Learning}
Reinforcement Learning aims to train an agent to solve tasks through direct interaction with the environment, but such interaction is often expensive or impractical in domains such as robotics and healthcare. To address this limitation, Offline Reinforcement Learning learns policies from a fixed dataset collected by a behavior policy \citep{offlinerl-tutorial,siegel2020doingworkedbehavioralmodelling,jaques2019wayoffpolicybatchdeep,pmlr-v119-agarwal20c,tree-based}.

Despite its promise, naively applying online RL algorithms in an offline setting often leads to severe performance degradation due to out-of-distribution actions and extrapolation errors \citep{Offline-RL,kumar2019stabilizing,offlinerl-tutorial}. Existing work mitigates these issues through a range of techniques, including $Q$-value regularization \citep{brac,CQL,critic-regularized-regression} and behavior cloning–based objectives \citep{TD3+BC}.

More recently, the fixed-dataset nature of Offline RL has motivated the adoption of Transformer-based architectures \citep{DT,edt,Q-Transformer}, which enable efficient and fully parallelized supervised training over trajectories.

\subsection{Conditioned Sequence Models}
Conditioned Sequence Models (CSMs) \citep{DT, tt} cast RL in the supervised learning paradigm by predicting actions from historical state-action pairs $(s_i, a_i)$ and return-to-go (RTG) tokens $\text{rtg}_i$.
In particular, the RTG is derived from the dataset as cumulative rewards during training, and is provided as a user-specified conditioning during inference.
Although recent theoretical results suggest CSMs can recover target returns under ideal assumptions
\citep{dt-theory, provable-ICL, gdt}, empirical studies \citep{DC, RADT} show that
the information carried by the RTG is often under-utilized by CSMs, leading to poor alignment between the target behavior and desired RTGs.

To address this issue, RADT \citep{RADT} enhances RTG sensitivity through architectural modifications, at the cost of substantial computational and parameter overhead, by introducing an additional attention layer in each transformer block.

\subsection{$Q$-Learning in Conditioned Sequence Models}
$Q$-functions are widely used to improve CSMs. Early methods like QDT \citep{qdt} and CGDT \citep{CGDT} leverage pretrained $Q$-functions for RTG relabeling and dataset-bias mitigation, while later approaches such as QT \citep{QT}, QCS \citep{QCS}, and TD3-ODT \citep{RLvitamin} backpropagate $Q$-value gradients through predicted actions.

While $Q$-gradient methods achieve competitive peak performance, they often push the policy toward maximum-value actions regardless of the target RTG, degrading alignment as it collapses to relatively high-return regions within the data distribution (\cref{fig:cmp}). 
Furthermore, although existing methods leverage $Q$-functions to improve CSMs, little attention has been paid to how CSMs can, in turn, benefit $Q$-function training through their alignment capabilities. 

\subsection{Multi-Task and Meta-Reinforcement Learning}
Meta-Reinforcement Learning (Meta-RL) \citep{MetaRLTutoral} aims to quickly adapt an agent to new tasks with similar underlying structures \citep{MAML,RL^2}. A more challenging setting, Offline Meta-RL, considers scenarios where the agent must learn from a fixed dataset and is expected to generalize to unseen test tasks \citep{MACAW,NEURIPS2021_24802454}. Owing to the in-context learning capabilities of Transformer architectures, recent works have increasingly deployed Transformers for such tasks \citep{prompt-dt,dpt}. 


\section{Preliminaries}

We consider a Markov Decision Process (MDP) defined by the tuple $(\mathcal{S}, \mathcal{A}, P, R, \gamma)$, where $\mathcal{S}$ and $\mathcal{A}$ represent the state and action spaces, $P(s_{t+1} | s_t, a_t)$ denotes the transition probability, $R(s, a)$ is the reward function, and $\gamma \in (0, 1]$ is the discount factor. Following the Decision Transformer framework \citep{DT}, we represent the input as a sequence of reward-to-go (RTG), state, and action tokens:
\begin{align*}
\boldsymbol\tau = (\text{rtg}_0, s_0, a_0, \ldots, \text{rtg}_H, s_H, a_H),
\end{align*}
where $\text{rtg}_t = \sum_{i=t}^H r_i$ represents the cumulative future reward at time $t$. This RTG signal serves as the target return to condition the policy. To ensure computational efficiency, we employ a $k$-step context window, where the truncated sequence at time $t$ is defined as:
\begin{align*}
\boldsymbol\tau_t = (\text{rtg}_{t-k+1}, s_{t-k+1}, a_{t-k+1}, \ldots, \text{rtg}_t, s_t, a_t).
\end{align*}

The RTG token serves as the primary mechanism for conditioning the model's behavior. 
Ideally, a DT-based model should represent a family of policies indexed by the target RTG \citep{dt-theory}, expressed as:
\begin{align}
    \Pi_{\mathrm{DT}} = \{\pi_{z} \mid z \in \mathbb{R}_+\},
\end{align}
In practice, the RTG tokens are usually controlled during inference to inform the desired behavior. To reflect this, we define a \textit{modified RTG sequence} $\boldsymbol\tau_t^{g}$ for any scalar $g \in \mathbb R$ by shifting all RTG tokens in the context window by $g$:
\begin{align}
\boldsymbol\tau_t^{g} = (\text{rtg}_{t-k+1}{+}g, s_{t-k+1}, a_{t-k+1}, \ldots, \text{rtg}_t{+}g, s_t, a_t). \label{eq:modified-rtg}
\end{align}
\section{Methods}
\noindent \textbf{Motivation.} To investigate the interpretability of target RTG, we benchmark several variants of the Decision Transformer \citep{DT, QT, RADT,DC} and analyze their behavior under various RTG conditioning values. As illustrated in \cref{fig:cmp}, we consistently observe a significant gap between the target RTG and the actual rollout performance. 

This misalignment is especially evident in the \textit{HalfCheetah} environment, where the model exhibits almost no sensitivity to variations in the requested RTG.
Such findings suggest that the RTG conditioning remains marginal in the model's decision process, directly motivating the development of \methodname{} to explicitly enforce RTG-behavior alignment.

\subsection{Training Model with RTG-to-Behavior Alignment }

From the above discussion, we aim to train a CSM variant that internalizes the \textbf{partial ordering between return-to-go (RTG) and action quality}. Ideally, this learning objective can be formulated as a constrained optimization problem:
\begin{equation}
\min_{\theta} L_{\text{SL}}(\theta), \quad \text{s.t. } \frac{\partial Q_{\psi}(s,\pi_{\theta}(s,\text{RTG}))}{\partial \text{RTG} }\ge 0,
\label{eq:constrained_obj}
\end{equation}
where $L_{\text{SL}}(\theta)$ is the standard supervised learning loss ensuring the policy remains anchored to the offline dataset, while the constraint 
ensures the $Q$-function of the predicted action is monotonically increasing with respect to the input RTGs.

To preserve this monotonicity under the variations of RTGs, directly calculating the gradient of critic $Q$ can be time-consuming due to the backpropagation through the $Q$ function and decision transformer. In order to improve the time efficiency, we consider the zero-th order estimation of the gradient following \citet{ES,SignSGD}. In particular, for any small enough perturbation $\delta$,
\begin{align*}
\frac{\partial Q_\psi(s, \pi_\theta(s, \text{RTG}))}{\partial \text{RTG}}
\approx 
\frac{
Q_\psi(s, \hat a^\delta) 
- Q_\psi(s, \hat a )
}{\delta},
\end{align*}
where $\hat a^\delta =\pi_\theta(s,\text{RTG}+\delta)$ and $\hat a = \pi_\theta(s,\text{RTG})$. To further mitigate the effect of the magnitude of $\delta$ and allow large perturbations $\delta$, we convert~\eqref{eq:constrained_obj} into the following objective:
\begin{align}
\min_{\theta} L_{\text{SL}}(\theta),~\text{s.t.}~\mathrm{sgn}(\delta)\big(Q_\psi(s, \hat a^\delta) - Q_\psi(s, \hat a )\big) \ge 0, \label{eq:revised-constrained-obj}
\end{align}

To this end, we solve~\eqref{eq:revised-constrained-obj} by introducing the Lagrange multiplier named as \textbf{alignment loss} $L_{\text{Align}}$.
In particular, given an input sequence $\boldsymbol \tau_t$ and its perturbed version $\boldsymbol \tau_t^\delta$ generated by adding a sequence-level noise $\delta \sim \mathcal{N}(0, \sigma_e^2)$, the alignment loss $L_{\text{Align}}$ is defined by:
\begin{equation}
L_{\text{Align}} = \textstyle{\sum_{i=t-k+1}^t} ~\ {I}_{\mathcal{C}} \cdot \big| Q_{\psi}(s_i, \hat{a}_i^\delta) - Q_{\psi}^\perp(s_i, \hat{a}_i) \big|,
\label{alignment-eq}
\end{equation}
where $\hat{a}_i$ and $\hat{a}_i^\delta$ are the $i$-th predicted actions conditioned on $\boldsymbol\tau_t$ and modified RTG sequence $\boldsymbol\tau_t^\delta$ defined by~\eqref{eq:modified-rtg}; $Q^\perp_{\psi}$ denotes the stop-gradient operator which is widely used in $Q$-based algorithms~\cite{RL-Intro}  treating the original $Q$-value as a fixed reference to stabilize training;  indicator function $\mathbb{I}_{\mathcal{C}}$ detects the constraint violation in~\eqref{eq:revised-constrained-obj}:
\begin{equation*}
\mathbb{I}_{\mathcal{C}} = 
\begin{cases} 
1, & \text{if } \text{sgn}(\delta) \big( Q_{\psi}(s_i, \hat{a}_i^\delta) - Q_{\psi}^{\perp}(s_i, \hat{a}_i) \big) < 0 \\
0, & \text{otherwise}
\end{cases}.
\end{equation*}
\rev{For simplicity, let $Q^\delta_i = Q_\psi(s_i,\hat a_i^\delta)$ and
$Q_i = Q^\perp_\psi(s_i,\hat a_i)$. The indicator activates when the RTG
perturbation and the induced critic change have inconsistent directions,
i.e., $\mathrm{sgn}(\delta)(Q^\delta_i-Q_i)<0$. This results in a directional
ranking penalty that corrects violations of the RTG--value monotonicity while
leaving already ordered pairs unchanged. In the context of offline RL, where
$Q$-functions trained on static datasets are prone to scale bias and local
inaccuracies~\citep{kumar2019stabilizing}, this yields a \textit{conservative
alignment} objective by enforcing relative ordering rather than exploiting
potentially inaccurate value magnitudes that could push the policy toward
out-of-distribution actions~\citep{Offline-RL}.}


Together with the joint state-action prediction with loss $L_{\text{SL}}(\theta)
= \textstyle{\sum_{i=t-k+1}^{t}}
\| s_i - \hat{s}_i \|^{2}
+ \| a_i - \hat{a}_i \|^{2}$ following the decision transformer's causal structure and input RTGs, this alignment objective encourages the model to internalize the RTG-action-$Q$ mapping and regularize the policy to respect a partial RTG ordering. To this end, the overall training objective can be summarized as
\begin{align}
\mathcal{L}_{\text{total}}(\theta) = L_{\text{SL}}(\theta) + \lambda_e \, L_{\text{Align}}(\theta),
\label{eq:ActorLoss}
\end{align}
where $\lambda_e$ controls the relative weight of the alignment constraint. This formulation allows \methodname{} to learn a coherent family of policies that do not only imitate the dataset, but are also systematically steerable via RTG conditioning.

\subsection{Training the $Q$ Function with \textit{RTG Perturbation}}

In \methodname{}, a $Q$-function is required to compute the alignment loss. 
If the critic is either pretrained solely on the behavior policy or naively co-trained using actions generated by the policy conditioned on $\boldsymbol \tau_t$, it tends to merely reflect the offline data distribution and remains anchored to suboptimal returns. 
When the critic remains stationary, the alignment loss defined in \eqref{eq:ActorLoss} becomes self-limiting, as it penalizes the policy for attempting to exceed the RTGs present in the static dataset. 
In practice, this mismatch not only limits alignment in high-RTG regions, but can also degrade alignment across the entire RTG spectrum (see~\cref{tab:ablation-fixq}).

To ensure the critic accurately reflects the policy's behavior across the RTG spectrum, we train the critic to satisfy Bellman consistency with respect to the current policy's response under RTG perturbations. Specifically, the critic $Q_\psi$ is optimized to minimize:
\begin{align}\textstyle{
L_q(\psi_{1, 2}) = \sum_{i=t-k+1}^{t-1} \sum_{m=1}^2 
\bigl(Q_{\psi_m}(s_i, a_i) - y_i' \bigr)^2 ,
\label{eq:Loss-Q}
}\end{align}
with the target value $y_i'$ being defined as:
\begin{align*}\textstyle{
y_i' = r_i + \gamma \min_{m=1,2} 
Q_{\psi_m'}^\perp\!(s_{i+1},\, \hat{a}_{i+1}^{\prime, \Delta\mathrm{RTG}}), 
}\end{align*}
where $\hat{a}_{i+1}^{\prime, \Delta\mathrm{RTG}}$ denotes the action predicted by the target policy $\pi_{\theta'}$ at the $(i{+}1)$-th step conditioned on the perturbed trajectory sequence $\boldsymbol \tau_t^{\Delta\mathrm{RTG}}$ and $Q_{\psi'}$ is the target critic. The offset $\Delta \mathrm{RTG}$ is a positive scalar that serves as a fixed trajectory-level perturbation.
Conditioning on the RTG-perturbed trajectory $\boldsymbol \tau_t^{\Delta \mathrm{RTG}}$, the policy $\pi_{\theta'}$ induces a higher-return action support for the critic $Q_\psi$, encouraging the critic to reflect higher-return behaviors within the RTG-conditioned policy family. 
As shown in our analysis  in Sec.~\ref{Analysis_perturb}, $\Delta \mathrm{RTG}$ in the co-training mechanism effectively acts as a control parameter for the critic learning dynamics. 

In particular, increasing $\Delta \mathrm{RTG}$ biases the RTG-conditioned policy toward higher-return actions, generating more informative targets for critic learning. Through the alignment loss, these improved value estimates are fed back to the actor, enabling a positive actor–critic feedback loop that facilitates policy improvement while remaining grounded in the support of the offline data, rather than merely fitting static, suboptimal return averages.

Combined with the \methodname{} and the \emph{RTG-perturbation} in co-training the $Q$ function with the decision transformer, we present the algorithm as \cref{alg:main} in Appendix~\ref{app:alg}.

\section{Theoretical Analysis}
In this section, we analyze the alignment error of CSMs and the policy behavior of \methodname~under data support constraints. We defer the detailed proof to Appendix~\ref{app:thm}.
\subsection{Why \methodname{} Reduces Alignment Error}
\label{Analysis_align_error}
Our analysis builds on the following finite sample analysis for DT presented in~\cite{dt-theory}:
\begin{lemma}[Theorem 1 and Corollary 3, \citealt{dt-theory}]
For a sample size $N$ and a finite policy class $\Pi$, and an MDP with horizon $H$, let $\hat{\pi} \in \Pi$ be the empirical risk minimizer that minimizes the training loss. With probability at least $1-\delta_p$, the total alignment error satisfies
\begin{align*} \text{RTG}_\text{tgt}-\mathbb{E}_{\boldsymbol \tau\sim\hat\pi_{\text{RTG}_\text{tgt}}}[\text{RTG}_\text{real}] \le \mathcal O\bigg(H^2 \frac{\log(|\Pi|/\delta_p)^{1/4}}{N^{1/4}}\bigg). 
\end{align*} 
\end{lemma}

Crucially, the bound shows that alignment error scales with $(\log |\Pi|)^{1/4}$. In unconstrained sequence modeling, the policy class $\Pi$ can be sufficiently large and therefore lead to significant alignment gap. In contrast, \methodname{} imposes an order-preserving bias that encourages $Q(s, \pi(s, g))$ to be monotonically increasing with respect to RTG $g$. As we will present in the following theorem, ensuring this monotonicity will effectively reduce the policy class $\Pi$.
\begin{theorem}[Policy Space Reduction in Discrete Case]
\label{thm:mono-order}
Consider a discrete state and action space $S, A$, given a set of RTG values $G=\{g_1 < \dots < g_{|G|}\}$. Let $\Pi_{\rm free}$ be the class of all deterministic policies $\pi: S \times G \to A$. Assume that for each state $s$, the function $Q(s,\cdot): A \to \mathcal{V}_s \subset \mathbb{R}$ induces an ordering over actions.
A policy $\pi$ is admissible in the constrained class $\Pi_{\rm mono}$ if it preserves this order across RTG levels. 
Then the log-complexity satisfies:
\begin{align*}
\log|\Pi_{\rm mono}| \le |S| |G| \log \frac{(|G| + |A| - 1)}{|G|} + |S| |G|.
\end{align*}
Furthermore, assuming each $Q$-value is shared by at most $C$ actions and the resolution of the value space $|\mathcal{V}_s|$ is of the same order as the RTG resolution $|G|$ (i.e., $|\mathcal{V}_s| = \Theta(|G|)$), we have $\log|\Pi_{\rm mono}| = \tilde O(|S| |G|)$.
\end{theorem}

Compared to the unconstrained case where $\log |\Pi_{\rm free}| = O(|S||G| \log |A|)$, Theorem~\ref{thm:mono-order} illustrates that \methodname{} utilizes the directional information of the $Q$ function to effectively eliminate the $\log |A|$ dependence in the complexity term. This significant reduction in the hypothesis space provides a theoretical intuition for why \methodname{} can significantly diminish alignment errors observed in practice.

\subsection{Policy Optimality under High-RTG Conditioning}
\label{Analysis_perturb}
As the alignment loss $L_{\text{Align}}$ enables the partial ordering between the RTG and action quality, we would like to examine how the learned policy behaves when conditioned on extremely large RTG values.
This setting is particularly important in practice, as inputting sufficiently large RTGs is a standard evaluation protocol for CSMs.

To better understand this regime, the following theorem presents an asymptotic equivalent behavior of \methodname{} and QT~\citep{QT}, where QT explicitly optimizes for action-value maximization under fixed RTG.

\begin{theorem}[Equivalence in High-RTG Regime]
For a fixed RTG $R$, let $\pi_\theta(s, R)$ be the policy learned by \methodname{} and $\pi_{\text{QT}}(s, R)$ be the policy that minimizes the following QT objective~\citep{QT} using a fixed $Q$ function:
\begin{align*}
\mathcal{L}_{\text{QT}}(\pi, R)
= \!\!\!\!\!\underset{(s,a) \sim \mathcal{D}}{\mathbb{E}}
[\|\pi(s,R) - a\|^2 - \eta Q(s, \pi(s,R))].
\end{align*}
Assuming the dataset $\mathcal{D}$ has full support over the action space, then as $R \to \infty$, \methodname{} recovers the behavior of an idealized QT that pursues maximum value:
\begin{align*}
    \lim_{R \to \infty} \pi_{\theta}(s, R) =  \lim_{R \to \infty}\pi_{\text{QT}}(s,R) = \arg\max_{a \in \mathcal{A}} Q(s,a).
\end{align*}
\label{Equivalence}
\end{theorem}
\begin{remark}
\cref{Equivalence} highlights the robust extrapolation of \methodname{} under extreme conditioning. When prompted with an abnormally large RTG, the alignment objective prevents the policy from collapsing into erratic behavior or fragile interpolations. Instead, \methodname{} reverts to a value-maximizing mode similar to QT, utilizing the guidance of $Q_\psi$ to select high-value actions.

\end{remark}

Then the next theorem further characterizes the effectiveness of the \textit{RTG-perturbation} techniques in learning the $Q$ function used in \cref{Equivalence}. Especially, the following theorem shows different asymptotic behaviors under different choices of $\Delta\mathrm{RTG}$ used in \textit{RTG-perturbation}.

\begin{theorem}[Impact of $\Delta\mathrm{RTG}$]
Let $\{Q_m, \pi_m\}_m$ be the sequences produced by \methodname{}. 
Consider the critic update using a perturbed conditioning signal 
$\tilde{R} = \mathrm{RTG} + \Delta\mathrm{RTG}$.
Then the following statements hold:

\noindent \textbf{(No Perturbation).} 
If $\Delta\mathrm{RTG}{=}0$, the updates remain within the behavior policy’s action support, leading to convergence of $\{Q_m\}_m$ to a value function close to $Q_\beta$, corresponding to conservative evaluation within the behavior support.

\noindent \textbf{(Large Perturbation).} 
If $\Delta\mathrm{RTG}$ is sufficiently large such that $\tilde{R}$ exceeds all returns in $\mathcal{D}$, then $\{Q_m\}_m$ converges to the optimal action-value function $Q^*$ and the induced policies $\{\pi_m( s, \tilde{R})\}_m$ converge to the corresponding optimal policy $\pi^*$, both restricted to the support of $\mathcal{D}$.

\label{optimal}
\end{theorem}

\begin{remark}
\cref{optimal} characterizes the effect of $\Delta\mathrm{RTG}$ on the behavior of the critic update: at the extremes when $\Delta\mathrm{RTG}$ approaches 0, \textit{RTG-perturbation} recovers a SARSA-style policy evaluation, while a sufficiently large $\Delta\mathrm{RTG}$ approximates standard $Q$-learning \citep{RL-Intro}. For intermediate values, the update interpolates between these behaviors, conceptually similar to the expectile parameter in IQL \citep{IQL}. This implies that $\Delta\mathrm{RTG}$ can be used to guide the learning of $Q$ by providing a more structured and higher-quality action space. In practice, while a larger $\Delta\mathrm{RTG}$ encourages the discovery of optimal policies, it can also increase optimization difficulty and potential instability; we therefore treat $\Delta\mathrm{RTG}$ as a tunable hyperparameter to balance policy improvement and training stability. We further explore the impact of different choices of $\Delta\mathrm{RTG}$ 
in Sec.~\ref{aba-rtg}.

\end{remark}

\begin{table*}[ht]
\revtab{
\centering
\caption{
Performance (D4RL normalized score $\uparrow$) of \methodname{} and other state-of-the-art baselines on Gym domains.
Results are averaged over five random seeds, and we report the mean $\pm$ standard error. Boldface numbers denote the highest or comparable scores among the algorithms.
}

\resizebox{\textwidth}{!}{
\begin{tabular}{lccccccccccc}
\toprule
Dataset & IQL & TD3+BC & DT & CGDT & LSDT & DC & DM & RADT & QT & QCS & \methodname  \\
\midrule
halfcheetah-medium-replay & 44.1 & 44.6 & 36.6 & 40.4 & 42.9 & 41.3 & 39.6 & 41.3 & 48.9 & 54.1 & \textbf{57.1} $\pm$ 0.74 \\
hopper-medium-replay      & 92.1 & 60.9 & 82.7 & 93.4 & 93.9 & 94.2 & 95.4 & 95.7 & \textbf{102} & 100.4 &\textbf{102.2}  $\pm$ 0.64\\
walker2d-medium-replay    & 73.7 & 81.8 & 79.4 & 78.1 & 74.7 & 76.6 & 85.5 & 75.9 & 98.5 & 94.1 & \textbf{101.3}  $\pm$ 0.73 \\
\midrule
halfcheetah-medium        & 47.4 & 48.3 & 42.6 & 43 & 43.6 & 43 & 43.5 & -- & 51.4 & 59 &\textbf{ 65.3} $\pm$ 0.63 \\
hopper-medium            & 63.8 & 59.3 & 67.6 & 96.9 & 87.2 & 92.5 & 98.1 & -- & 96.9 & 96.4 & \textbf{102.1}  $\pm$ 0.74 \\
walker2d-medium           & 79.9 & 83.7 & 74 & 79.1 & 81 & 79.2 & 83.8 & -- & 88.8 & 88.2 & \textbf{94.7 } $\pm$ 0.67 \\
\midrule
halfcheetah-medium-expert & 86.7 & 90.7 & 86.8 & 93.6 & 93.2 & 93 & 93.9 & 93.1 & 96.1 & 93.3 &\textbf{98.8}  $\pm$ 0.68 \\
hopper-medium-expert      & 91.5 & 98 & 107.6 & 107.6 & 111.7 & 110.4 & 111.8 & 110.4 & \textbf{113.4} & 110.2 &\textbf{114.0}  $\pm$ 0.18 \\
walker2d-medium-expert    & 109.6 & 110.1 & 108.1 & 109.3 & 109.8 & 109.6 & 112.7 & 109.7 & 112.6 & 116.6 & \textbf{121.4 } $\pm$0.52  \\
\midrule

Sum & 688.8 & 677.4 & 685.4 & 741.4 & 738 & 739.8 & 764.3 & -- & 808.6 & 812.3 & \textbf{856.9}\\
\bottomrule
\end{tabular}
}
\label{tab:gym-results1}
}
\end{table*}

\begin{table*}[ht]
\revtab{
\caption{Performance (D4RL normalized score $\uparrow$) of \methodname~and baseline methods on AntMaze. Results averaged over five seeds. Boldface numbers denote the highest or comparable scores among the algorithms.}
\resizebox{\textwidth}{!}{%
\begin{tabular}{lcccccccccc}
\toprule
Dataset & IQL & TD3+BC& CQL & DT&CGDT &RVS& DC & QT & QCS &  \methodname \\
\midrule
antmaze-umaze  & 87.5 &78.6 &  74   & 65.6&71 &64.4& 85.0 & 90.7& 92.7 & \textbf{96.4}$\pm 1.12$ \\
antmaze-umaze-diverse & 62.2  & 71.4  & \textbf{84}   & 51.2 &71&70.1 &78.5 & \textbf{83.7}   & 72.3& 73.2 $\pm 6.44$  \\
antmaze-medium-play   & 71.2  & 10.6  & 61.2 & 4.3 &-- &4.5& 33.2   & 78.6 & 81.6   & \textbf{85.6} $\pm 4.74$  \\
antmaze-medium-diverse & 70  & 3 & 53.7 & 1.2&--&7.7  & 27.5   & 62.3 & \textbf{79.5} & 76.2 $\pm 6.62$  \\
\bottomrule
\end{tabular}
}
\label{tab:gym-results2}
}
\end{table*}

\section{Experiments}
We evaluate the \methodname{} mainly on D4RL benchmark \cite{d4rl} including the \texttt{Gym} (Hopper, HalfCheetah, Walker2d) and \texttt{AntMaze} (umaze, medium) tasks.

\subsection{Implementation Details}

We compare \methodname{} with a diverse set of baselines, including both value-based methods
and CSM methods.
For value-based approaches, we consider IQL~\citep{IQL}, TD3+BC~\citep{TD3+BC}, and CQL~\citep{CQL}.
For CSM-based methods, we include DT~\citep{DT}, DC~\citep{DC}, RVS~\citep{rvs}, CGDT~\citep{CGDT}, LSDT~\citep{LSDT},
DM~\citep{DM}, RADT~\citep{RADT}, QT~\citep{QT}, and QCS~\citep{QCS}.

For training stability, we pretrain the $Q$-function on the offline dataset. To decouple the performance of RTG-to-behavior alignment from the effects of critic pretraining, we employ a standard Double $Q$-learning update \citep{doubleq} across most environments, with the exception of \texttt{antmaze}, where IQL \citep{IQL} is adopted due to its superior performance in sparse-reward tasks (see Appendix~\ref{Pretrain} for details).

Motivated by prior observations that vanilla attention underutilizes RTG tokens \citep{DC,RADT}, we empirically introduce a lightweight convolutional projection for attention inputs. As shown in \cref{tab:ablation-conv}, this does not affect peak performance but improves alignment behavior; architectural details and ablations are provided in Appendix~\ref{app:model_structure}.

\subsection{Metric}
We follow standard offline RL evaluation and report D4RL normalized scores averaged over five seeds, with 50 rollouts per seed (100 for AntMaze).

To evaluate \emph{alignment} error, we measure how accurately a policy tracks a prescribed RTG target.
For each environment, we sweep target RTGs from the minimum to maximum D4RL return in increments of $100$, execute $30$ rollouts per target, and record the resulting normalized scores.
Let $\text{score}_{\mathrm{tgt},j}$ and $\text{score}_{\mathrm{real},j}$ denote the $j$-th target and achieved scores. For $n$ target RTGs, we evaluate the alignment error by the root mean squared deviation
\begin{equation}
M = \sqrt{\tfrac{1}{n} \textstyle{\sum_{j=1}^{n}}
(\text{score}_{\mathrm{real}, j}
- \text{score}_{\mathrm{tgt}, j} )^{2}},
\end{equation}
where lower $M$ indicates better RTG-to-behavior alignment.

\subsection{Main Results}
We first present the performance of \methodname{} on Gym domains and AntMaze tasks in \cref{tab:gym-results1,tab:gym-results2}. The results show that \methodname{} consistently achieves state-of-the-art performance, and in many cases surpasses existing methods across the datasets. This demonstrates that, although our method is primarily motivated by alignment, it can attain high scores when conditioned on high RTG, as is typical for CSMs. \looseness=-1

We also report alignment metrics in \cref{tab:alignment}. 
\methodname{} achieves strong alignment across the evaluated Gym domains, validating the effectiveness of our proposed approach. Notably, in HalfCheetah, where many CSMs tend to produce similar behaviors for different RTGs, our method reliably aligns with the target RTG. 

\begin{table*}[ht]
\centering
\revtab{
\caption{
Alignment performance (Root Mean Squared Error $\downarrow$) of \methodname{} and other state-of-the-art baselines on Gym domains.
Results are averaged over five random seeds. 
}
\resizebox{\textwidth}{!}{%

\begin{tabular}{lcccccc}
\toprule
Dataset & DC & RADT & DT &QT&QCS &\methodname  \\
\midrule

walker2d-medium-replay    & 14.05 $\pm$0.11 & 14.27$\pm$ 2.09& 18.67 $\pm$ 1.08&35.93 $\pm 1.9$& 20.55 $\pm$ 3.26&\textbf{6.39} $\pm$ 0.51  \\
hopper-medium-replay      & 12.65$\pm$0.3 & \textbf{6.49} $\pm$0.16 & 11.38 $\pm$ 1.31 &39.69 $\pm $ 7.4&27.49 $\pm$ 3.09&10.1 $\pm$ 1.2   \\
halfcheetah-medium-replay & 28.67 $\pm$0.34& 50.15$\pm$ 1.45& 30.93 $\pm$0.59& 28.25$\pm$ 0.35&28.12$\pm$0.52 &\textbf{17.8} $\pm$ 0.53  \\
\midrule

walker2d-medium           & 16.72 $\pm$ 3.77& 18.68 $\pm$4.46& 30.69 $\pm$ 1.99 &38.25 $\pm$ 7.99&40.85 $\pm$2.95 &\textbf{6.25} $\pm$ 0.43   \\
hopper-medium             &16.03$\pm$2.45 & 14.99$\pm$0.74 & 16.08$\pm$ 0.38&31.42$\pm$ 1.92 & 32.19$\pm$ 0.51&\textbf{8.39} $\pm$ 2.3  \\
halfcheetah-medium        & 28.89$\pm$ 0.1& 25.23$\pm$1.21& 29.22$\pm$0.28 &28.07$\pm 0.32$ &28.88$\pm$ 0.1&\textbf{12.36} $\pm$ 0.7 \\
\midrule

walker2d-medium-expert    & 13.62$\pm$ 1.59& 10.15 $\pm$0.69& 18.88 $\pm$1.07&63.71 $\pm$ 1.09& 48.37 $\pm$ 0.44&\textbf{6.84} $\pm$ 0.83   \\
hopper-medium-expert      & 12.22 $\pm$ 0.8& \textbf{8.18}$\pm 0.81$  & 16.92 $\pm$ 1.89&21.39 $\pm$ 0.1&24.91 $\pm$0.41 &11.7 $\pm$ 2.05  \\
halfcheetah-medium-expert & 17.88 $\pm$ 0.17& 17.55 $\pm$ 0.1& 19.79$\pm$1.74 &21.24 $\pm$ 0.96& 43.95 $\pm$1.1 &\textbf{1.24} $\pm$0.19    \\
\midrule

Mean & 17.86 & 18.58 & 21.39 &34.22&32.81 &\textbf{9.01 } \\
\bottomrule
\end{tabular}
}
\label{tab:alignment}
}
\end{table*}

\subsection{Ablation Study}
\label{aba-rtg}
In this section, we present key ablation studies for \methodname{}. 
Additional experiments, including hyperparameter variations, alternative indicator functions, architectural components, and training procedures, are reported in Appendix~\ref{more-aba}.

\noindent \textbf{Selection of $\Delta \mathrm{RTG}$ in \textit{RTG-perturbation}.}
To evaluate the impact of different $\Delta \mathrm{RTG}$ values, we conduct experiments on 
\rev{\texttt{halfcheetah-medium} with normalized $\Delta\mathrm{RTG} \in \{0, 1, 3, 5, 10\}$ (i.e., actual RTG normalized by $1000$).}
We report both the D4RL score and the alignment error in \cref{fig:rtg-offset}. 
Consistent with our analysis, increasing $\Delta\mathrm{RTG}$ generally improves overall performance, 
while the alignment error increases mildly.

\begin{figure}[ht]
\centering
\includegraphics[width=\linewidth]{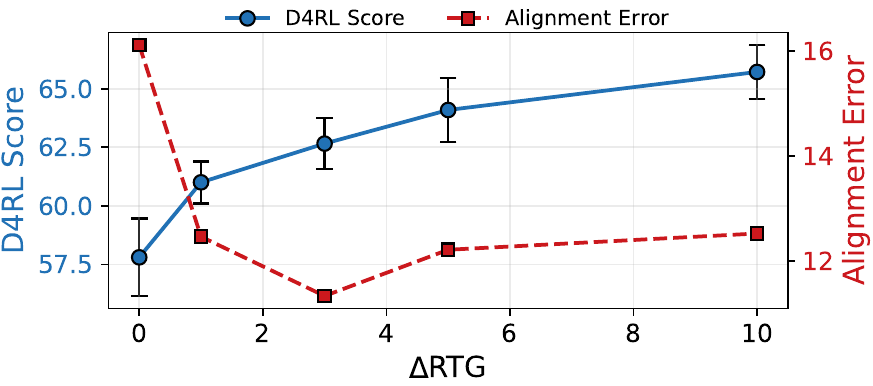}
\caption{Effect of normalized RTG offset $\Delta\text{RTG}$ (actual RTG divided by 1000) on policy performance (D4RL score) and alignment error in \texttt{halfcheetah-medium}.}
\label{fig:rtg-offset}
\end{figure}

Interestingly, alignment error increases as the RTG offset approaches zero, with most degradation occurring in the low-RTG regime. This effect is closely related to early-timestep collapse and is exacerbated when the $Q$-function is fixed after pretraining (Appendix~\ref{app:alignment-analysis}). We hypothesize that this occurs because the critic remains close to $Q_\beta$. In low-RTG regions, this leads to weak and noisy guidance due to the highly multimodal nature of low-return behaviors and their limited value separation. Consequently, RTG perturbations are important not only for high-RTG performance, but also for maintaining reliable alignment across the RTG spectrum.

On the other hand, larger RTG offsets target higher-return behaviors but also increase distributional shift, potentially destabilizing $Q$-function learning. To mitigate this, we restrict $\Delta\mathrm{RTG}$ to a moderate range; detailed values and experimental settings are reported in Appendix~\ref{Hyperparameters}.

\paragraph{Effect of the Convolution Layer.}
We ablate the 1D convolution layer in our architecture by reverting the model to the standard Decision Transformer. As shown in \cref{tab:ablation-conv}, while the convolution layer has a marginal impact on the best achievable performance, it substantially reduces the alignment error. A possible explanation is that the convolution operation enables the attention mechanism to better extract and propagate information from the RTG token, thereby providing more informative alignment signals.
\begin{table}[ht]
\caption{
Effect of the convolution layer.
Align. RMSE ($\downarrow$) denotes the alignment error, Perf. ($\uparrow$) denotes the overall performance (D4RL normalized score).
Results averaged over three seeds.
}
\centering
\resizebox{.95\linewidth}{!}{
\begin{tabular}{lcccc}
\toprule
Dataset &
\multicolumn{2}{c}{Align. RMSE $\downarrow$} &
\multicolumn{2}{c}{Perf. $\uparrow$} \\
\cmidrule(lr){2-3} \cmidrule(lr){4-5}
 & Ours & w/o Conv. & Ours & w/o Conv. \\
\midrule
Medium-Replay & 11.31 & 13.37 & 86.45 & 86.88 \\
Medium  & 9.19  & 14.22 & 87.70 & 87.21 \\
Medium-Expert & 6.80  & 10.70 & 111.37 & 110.94 \\
\bottomrule
\end{tabular}
}
\label{tab:ablation-conv}
\end{table}

\section{Behavioral Analysis and Generalization }

While D4RL rewards are usually scalar, these rewards often aggregate multiple components. For instance, \texttt{halfcheetah} uses the difference between the forward speed \emph{forward\_reward} and \emph{ctrl\_cost} as the reward. 
Therefore, it's natural to ask whether the model can meaningfully adjust its behavior conditioned on different target RTGs even when trained only on scalar returns.

To answer this, we evaluate rollouts across a range of target RTGs. Higher RTGs consistently induce faster locomotion, while lower RTGs lead to more cautious movement (\cref{fig:mean-velocity}). A finer-grained analysis of instantaneous velocity along representative
trajectories (\cref{fig:diff-traj}) shows that the agent adaptively adjusts gait and speed rather than following a single policy to satisfy the specified reward constraints. These results suggest that
\methodname{} has captured a semantically meaningful mapping from RTG tokens
to the underlying action space. Additional \texttt{ant} experiments are
provided in Appendix~\ref{Ant-Analysis} with similar observations.

\begin{figure}[ht]
\centering
\includegraphics[width=\linewidth]{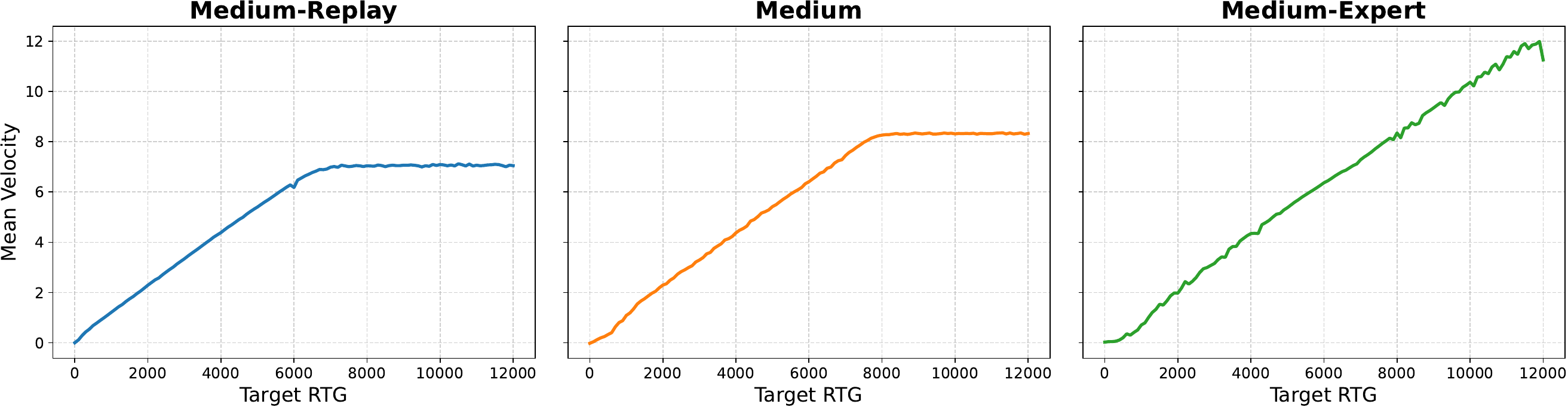}
\caption{Mean forward velocity of the agent conditioned on varying target RTGs across different datasets on \texttt{halfcheetah}.}
\label{fig:mean-velocity}
\end{figure}

\begin{figure}[htbp]
\centering
\includegraphics[width=\linewidth]{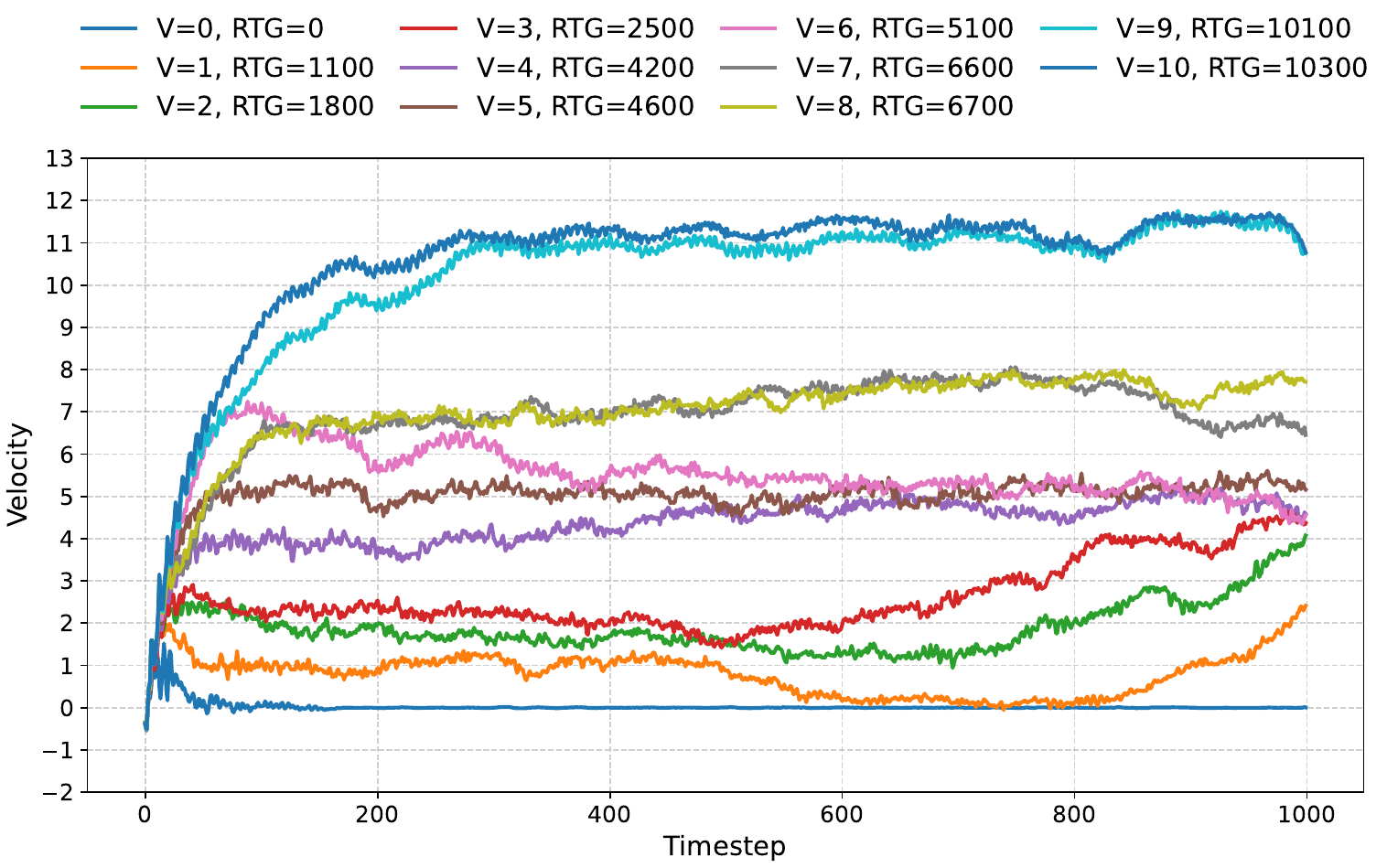}
\caption{Agent velocity over timesteps for different target RTGs on \texttt{halfcheetah-medium-expert}, averaged over 30 runs.}
\vspace{-1em}
\label{fig:diff-traj}
\end{figure}

\noindent \textbf{Zero-shot transfer to Meta-RL tasks.} With this observation, we further challenge the model's alignment robustness through a zero-shot transfer to \texttt{halfcheetah-vel} \citep{MACAW}. We compare our approach against two distinct categories of baselines: (i) \textbf{In-domain Meta-RL methods}, including Prompt-DT \citep{prompt-dt}, MACAW \citep{MACAW}, and PEARL \citep{PEARL}, which are trained on the \texttt{halfcheetah-vel} task family across multiple tasks; and (ii) \textbf{Cross-domain CSMs}, which, like \methodname{}, are trained exclusively on the \texttt{halfcheetah-medium-expert} dataset and generalize to the velocity-tracking task directly.

\begin{table}[!h]
\centering
\caption{
\textbf{Performance on HalfCheetah-vel.}
Meta-RL methods are trained on \texttt{HalfCheetah-vel},
while offline RL methods are trained on \texttt{HalfCheetah-medium-expert}
and evaluated via zero-shot transfer.
For Prompt-DT, $K^*{=}5$ follows the original setting with 5-step expert prompts used at both training and inference, while $K^*{=}0$ removes prompts entirely to enable a fair comparison with prompt-free CSMs.
}

\label{tab:halfcheetah-vel}
\resizebox{.95\linewidth}{!}{
\begin{tabular}{l c}
\toprule
\textbf{Method} & \textbf{Average Episode Return} \\
\midrule
\multicolumn{2}{c}{\cellcolor{gray!20}\textit{Trained on HalfCheetah-vel}} \\

\rowcolor{gray!20} MACAW \citep{MACAW} & -121.6 \\
\rowcolor{gray!20} PEARL \citep{PEARL} & -273.6 \\
\rowcolor{gray!20} Prompt-DT \citep{prompt-dt} ($K^*=0$) & -188.4 \\
\rowcolor{gray!20} Prompt-DT ($K^*=5$) & -38.6 \\

\midrule

\multicolumn{2}{c}{\textit{Trained on HalfCheetah-medium-expert}} \\

RADT \citep{RADT} & -754.7 \\
DC \citep{DC} & -756.7 \\
DT \citep{DT}& -764.6 \\
QCS \citep{QCS}& -1298.7\\
QT \citep{QT} & -824.6 \\
\methodname{} & \textbf{-142.3} \\

\bottomrule
\end{tabular}
}
\end{table}

To account for differences in reward formulation, observation dimensions, and
episode lengths, we apply a unified adjustment protocol to all
cross-domain CSMs (details in Appendix~\ref{Halfcheetah-vel-Transfer}).
Specifically, for each target velocity, the agent's behavior is controlled
solely by the input RTG, which is determined using simple linear interpolation.

Surprisingly, we found that \methodname{} exhibits a strong zero-shot transferability compared to other CSMs and achieves competitive performance even against in-domain Meta-RL baselines as shown in \cref{tab:halfcheetah-vel}. This demonstrates its superior ability to represent and generalize multiple RTG-conditioned policies across environments with significant discrepancies. Furthermore, we evaluate the model on higher target velocities that exceed the original range of the \texttt{halfcheetah-vel} benchmark in \cref{tab:halfcheetah-vel-rtg-high}. These results highlight the controllability of \methodname{} and its ability to generalize to a wide range of target velocities.

\section{Conclusion}
In this work, we introduce \methodname{}, a framework designed to align conditioned sequence models with different return-to-go (RTG) targets.
We analyze the underlying mechanisms of \methodname{} and show that actions from alignment-consistent CSMs can provide informative signals for $Q$-function learning.
Extensive experiments demonstrate that our approach learns a coherent family of policies and achieves state-of-the-art performance in both return maximization and RTG alignment across a wide range of environments.
\section*{Acknowledgements}
We thank the anonymous reviewers for their helpful comments. This research was supported by
WZ’s startup funding provided by the School of Data Science and Society at UNC Chapel Hill.

\section*{Impact Statement}

This paper presents work whose primary goal is to advance the field of machine learning, specifically in offline reinforcement learning and conditioned sequence models.
Our contributions focus on improving the alignment and controllability of learned policies using fixed datasets, without introducing new data sources or deployment mechanisms.

The techniques studied in this work are general algorithmic methods and are evaluated exclusively in standard simulated benchmark environments.
As such, we do not anticipate direct negative societal impacts arising uniquely from this research beyond those already well established for reinforcement learning methods in general.
Potential applications of improved policy alignment include enabling safer and more controllable decision-making systems, subject to appropriate domain-specific considerations.

\rev{
\section*{Limitations and Future Work}
This paper studies the RTG mismatch issue in return-conditioned offline RL and
proposes Q-guided alignment as a practical mitigation. Our experiments focus
on standard Gym and AntMaze benchmarks, and extending the evaluation to
broader domains such as visual control and robotic manipulation is an
important direction for future work. Our theoretical analysis uses simplifying
assumptions to isolate the effect of RTG--behavior alignment; extending the
analysis to more general settings remains open. Finally, while \methodname{}
shows empirical benefits in sparse-reward environments such as AntMaze, its
stable performance in these settings currently relies on additional
initialization or preprocessing for obtaining a reliable critic signal.
Developing a more robust sparse-reward variant is an interesting direction for
future work.
}


\bibliography{example_paper}
\bibliographystyle{icml2026}
\clearpage
\appendix
\onecolumn  
\renewcommand{\thetable}{\Alph{section}.\arabic{table}}
\renewcommand{\thefigure}{\Alph{section}.\arabic{figure}}

\section{Algorithm}
\label{app:alg}
In this section, we provide the formal pseudocode for \methodname. The training procedure is designed to ensure that the policy $\pi_\theta$ not only clones the behavioral data but also adheres to directional value consistency through the alignment loss.
\begin{algorithm}[H]
\caption{Training Procedure}
\begin{algorithmic}[1]
\label{alg:main}
\REQUIRE Dataset $\mathcal{D}$, sequence horizon $k$, training epochs $N$, batch size $B$, RTG offset $\Delta \text{RTG}$, noise variance $\sigma_e^2$, target update rate $\alpha$, pretrained critic networks $Q_{\psi_1}, Q_{\psi_2}$, target critic networks $Q_{\psi_1'}, Q_{\psi_2'}$
\ENSURE Trained actor parameters $\theta$

\STATE Initialize actor $\pi_{\theta}$ and target actor $\pi_{\theta'}$

\FOR{$\text{epoch} = 1$ to $N$}
    \STATE Sample batch of trajectories $\boldsymbol \tau \sim \mathcal{D}$
    \STATE Sample sub-trajectory $\boldsymbol \tau_t$ of length $k$ with random timestep $t$
    
    \text{// Critic update}
    \STATE Sample target actions $\hat a^{\prime, \Delta \mathrm{RTG}} \sim \pi_{\theta'}(\boldsymbol \tau_t^{\Delta \text{RTG}})$
    \STATE Update critic networks using \eqref{eq:Loss-Q}
    
     \text{// Actor update}
    \STATE Predict actions $\hat a \sim \pi_{\theta}(\boldsymbol \tau_t)$
    \STATE Sample noise $\delta \sim \mathcal{N}(0, \sigma_e^2)$
    \STATE Predict noisy actions $\hat a^\delta \sim \pi_{\theta}(\boldsymbol \tau_t^\delta)$
    \STATE Update actor using \eqref{eq:ActorLoss}
    
    \STATE Update target networks (actor and critics) with rate $\alpha$
\ENDFOR

\STATE \textbf{Return} $\theta$

\end{algorithmic}
\end{algorithm}

\section{Mathematical Details}
\label{app:thm}

\subsection{Proof of \cref{thm:mono-order}}
In order to prove \cref{thm:mono-order}, we first give a formal statement and analysis of results from \citet{dt-theory}.

\begin{theorem}
\label{thm:dt1}
Consider an MDP, a behavior policy $\beta$, and a conditioning function $f$ that is consistent with the reward dynamics, i.e.,
\begin{align*}
f(s) = f(s') + r(s,a)
\end{align*}
for any transition from $s$ to $s'$ under action $a$. Let
\begin{align*}
g(\boldsymbol  \tau) = \sum_{i=0}^H r_i
\end{align*}
denote the cumulative reward of a trajectory $\boldsymbol  \tau$. Assume:
\begin{enumerate}
    \item \textbf{Return Coverage:} the policy $\beta$ covers all initial states $s_1$, i.e.,
    \begin{align*}
        P(g(\boldsymbol  \tau) = (f(s_1) \mid s_1) \ge \alpha_f, \quad \forall s_1.
    \end{align*}
    \item \textbf{Near Determinism:} the environment satisfies
    \begin{align*}
        P(r \neq R(s,a) \mid s,a) \le \epsilon, \quad
        P(s' \neq T(s,a) \mid s,a) \le \epsilon
    \end{align*}
    for some reward function $r$ and transition function $T$.
\end{enumerate}
Then the expected return of the conditioned policy trained on infinite data $\pi_f$ satisfies
\begin{align*}
\mathbb{E}_{s_1}[f(s_1)] - \mathbb{E}_{\boldsymbol  \tau \sim \pi_f}[g(\boldsymbol  \tau)] \le \epsilon \left(\frac{1}{\alpha_f} + 2 \right) H^2,
\end{align*}
where $H$ is the horizon.
\end{theorem}
The result follows directly from Theorem 1 of \citet{dt-theory}. \qed

With \cref{thm:dt1}, we can analyze why previous CSMs struggle to align with target RTG values that deviate from the training distribution. Following \citet{RLvitamin}, we consider the case where $f(s_1) = \text{RTG}_{\text{tgt}}$.
According to \cref{thm:dt1}, the alignment error is bounded by:
\begin{align*}
\text{RTG}_{\text{tgt}} - \mathbb{E}_{\boldsymbol  \tau \sim \pi(\cdot|s,\text{RTG}_{\text{tgt}})}[\text{RTG}_{\text{real}}] 
\le \epsilon \left(\frac{1}{\alpha_f} + 2 \right) H^2,
\end{align*}
While \cref{thm:dt1} provides an analysis of DT's behavior in the infinite-data regime, we are more interested in the finite-data setting. In particular, we have:

\begin{theorem}
\label{thm:dt2}
Under the assumptions of \cref{thm:dt1}, let the training sample size be $N$, and further assume that 
\begin{itemize}
    \item $\frac{P_{\pi_f}(s)}{P_{\pi_\beta}(s)} \le C_f$ for all $s$.
    \item The policy class $\Pi$ is finite.
    \item $|\log \pi(a|s,g) - \log \pi(a'|s',g')| < c$ for all $(a,s,g,a',s',g')$ and all $\pi \in \Pi$.
    \item The approximation error of $\Pi$ is bounded by $\epsilon_\text{error}$.
\end{itemize}
Then, with probability at least $1-\delta_p$,
\begin{align*}
\mathbb{E}_{\boldsymbol  \tau \sim \pi(\cdot|s,\text{RTG}_{\text{tgt}})}[\text{RTG}_{\text{real}}] 
- \mathbb{E}_{\boldsymbol  \tau \sim \hat \pi(\cdot|s,\text{RTG}_{\text{tgt}})}[\text{RTG}_{\text{real}}] 
\le O\Bigg(\frac{C_f}{\alpha_f} H^2 \Big(\sqrt{c} \Big(\frac{\log |\Pi| / \delta_p}{N}\Big)^{1/4}\Big) + \sqrt{\epsilon_\text{error}}\Bigg).
\end{align*}
\end{theorem}
where $\hat{\pi} \in \Pi$ is the policy that minimizes the KL divergence  on the training dataset.

The result follows directly from Corollary 3 in \citet{dt-theory}. \qed

Combining \cref{thm:dt1} and \cref{thm:dt2},  we conclude that with probability at least $1-\delta_p$, the total alignment error of the learned policy $\hat \pi$ scales as

\begin{align}
\text{RTG}_\text{tgt}
- \mathbb{E}_{\boldsymbol  \tau \sim \hat \pi(\cdot|s,\text{RTG}_{\text{tgt}})}[\text{RTG}_{\text{real}}] 
\le O\Bigg(
\epsilon \tfrac{1}{\alpha_f} H^2+
\frac{C_f}{\alpha_f} H^2 \Big(\sqrt{c} \Big(\frac{\log |\Pi| / \delta_p}{N}\Big)^{1/4}\Big) + \sqrt{\epsilon_\text{error}}\Bigg).
\end{align}

By imposing a monotonicity constraint on the $Q$-value of the policy output with respect to the given RTG, \methodname{} effectively restricts the policy space $\Pi$. We then give the proof of \cref{thm:mono-order}:

\begin{proof}[Proof of \cref{thm:mono-order}]
In the unconstrained case, each state-RTG pair $(s,g)$ can be independently mapped to any of $|A|$ actions. With $|S|$ states and $K$ RTG levels, the total number of policies is:
\begin{align*}
|\Pi_{\rm free}| = |A|^{|S| |G|}, \quad \log |\Pi_{\rm free}| = |S| |G| \log |A|.
\end{align*}

In the constrained case $\Pi_{\rm mono}$, for a fixed state $s$, the sequence 
$(\pi(s,g_1),\dots,\pi(s,g_{|G|}))$ must be non-decreasing with respect to $\preceq_s$ induced by the scalar-valued $Q(s,\cdot)$.
In the presence of ties in $Q(s,\cdot)$, we assume an arbitrary but fixed tie-breaking rule,
which induces a total order $\preceq_s$ compatible with $Q(s,\cdot)$.
Treating each state independently, the number of monotone sequences of length $K$ over the set of actions $A$ is given by the multiset coefficient:

\begin{align*}
|\Pi_{{\rm mono},s}| = \binom{|G|+ |A| - 1}{|G|}.
\end{align*}

Across all states:
\begin{align*}
|\Pi_{\rm mono}| = \prod_{s \in S} |\Pi_{{\rm mono},s}|, \quad
\log |\Pi_{\rm mono}| = \sum_{s \in S} \log \binom{|G|+ |A| - 1}{|G|}.
\end{align*}

Using $\binom{n}{k} \le (en/k)^k$, we have
\begin{align*}
\log |\Pi_{\rm mono}| \le |S| |G| \log \frac{e(|G| + |A| - 1)}{|G|}.
\end{align*}

If we further consider the regime where the resolution of the score space
$|\mathcal{V}_s|$ is comparable to the RTG discretization $|G|$
(i.e., $|\mathcal{V}_s| = \Theta(|G|)$), then, since each score value
can be shared by at most a constant number $C$ of actions, we have
$|A| \le C |\mathcal{V}_s| = \Theta(|G|)$. Therefore,
\[
\frac{|G| + |A| - 1}{|G|} = O(1).
\]
Consequently, the resulting log-complexity bound satisfies
\[
\log|\Pi_{\rm mono}| = O(|S| |G|),
\]
eliminating any explicit logarithmic dependence on $|A|$.

\begin{remark}[On the Comparability of RTG and Score Resolution]
The assumption $|\mathcal{V}_s| = \Theta(|G|)$ is reasonable in our setting.
RTG corresponds to undiscounted returns ($\gamma=1$), while the critic $Q_\psi$ provides a discounted estimate of action quality; both therefore encode return-based preferences at comparable resolution.
Moreover, RTG perturbations are sampled with non-negligible variance (see \cref{tab:hyper}), encouraging consistent policy responses over a broad RTG range rather than infinitesimal changes.
\end{remark}

\end{proof}

\subsection{ Proof of \cref{Equivalence}}
We consider an MDP 
\begin{align*}
M = (\mathcal{S}, \mathcal{A}, P, R, \gamma)
\end{align*} 
with horizon $H$, where  $P(s' \mid s, a)$ denotes the transition probability. 
Let $R^*$ be a sufficiently large total reward, larger than the return of any trajectory 
\(\boldsymbol  \tau = (s_0, a_0, r_0, \dots, s_H, a_H, r_H)\) in the environment. 

Our proof strategy is to decompose \cref{Equivalence} into two parts, formally stated as two lemmas. These lemmas show that both QT and \methodname{}'s objectives converge to the same greedy action-selection criterion when the target RTG is sufficiently large.

\begin{lemma}
Let $Q$ be a fixed critic defined for all $(s,a)$ pairs in the dataset, and let $\pi^{\mathrm{Align}}$ denote the policy that minimizes the alignment actor loss in \eqref{eq:ActorLoss}.  
Assume the dataset fully covers all state-action pairs, i.e.,
\begin{align*}
p_{\mathcal{D}}(a \mid s) > 0, \quad \forall s \in \mathcal{S}, \; \forall a \in \mathcal{A},
\end{align*}
where $p_{\mathcal{D}}(a \mid s)$ denotes the empirical probability of observing action $a$ at state $s$ in the dataset.
Then, for the maximal target return $R^*$, the policy $\pi^{\mathrm{Align}}_{R^*}$ outputs an action
\begin{align*}
a^* = \arg\max_{a \in \mathcal{A}_{\mathrm{data}}(s)} Q(s,a),
\end{align*} 
for each state $s$.
\label{Q-align-analysis}
\end{lemma}

\begin{proof}
To analyze the behavior of \methodname{} as $\text{RTG}_{\text{tgt}} \to R^*$, we focus on the components of the loss that are sensitive to the target return. For a fixed state $s$, the alignment actor loss at $R^*$ is defined as:
\begin{align*}
L_{\mathrm{total}}(s,R^*) &= \mathbb{E}_{a \sim \mathcal{D}_{R^*}(s)} \|\pi_{R^*}(s) - a\|^2 + \mathbb{E}_{\text{RTG} \sim \mathcal{D}(s)} \mathbb{I}\bigl( Q(s, \pi_{\text{RTG}}(s)) > Q(s, \pi_{R^*}(s)) \bigr) \\
&= \underbrace{\mathbb{E}_{a \sim \mathcal{D}_{R^*}(s)} \|\pi_{R^*}(s) - a\|^2}_{\text{supervised regression term}} + \underbrace{\mathbb{E}_{a' \sim \mathcal{D}(s)} \mathbb{I}\bigl( Q(s, a') > Q(s, \pi_{R^*}(s)) \bigr)}_{\text{ranking term}},
\end{align*}
where $\mathbb{I}(\cdot)$ is the indicator function. The second equality holds under the assumption that the model successfully minimizes the supervised loss within the data distribution, such that $\pi_{\text{RTG}}(s) \approx a'$ for $(s, a', \text{RTG}) \sim \mathcal{D}$.
\begin{remark}
The assumption $\pi_{\text{RTG}}(s) \approx a'$ can be satisfied alongside $L_{\text{Align}}$ because the two objectives are compatible: $L_{\text{SL}}$ constrains the policy at discrete dataset points, while $L_{\text{Align}}$ only enforces directional monotonicity with respect to RTG. Since the alignment loss vanishes once the ordering is satisfied, it does not conflict with the supervised targets. In the high-dimensional policy space, a policy exists that simultaneously respects $L_{\text{SL}}$ and the global monotonic structure of $L_{\text{Align}}$, assuming the dataset is consistent.
\end{remark}

In the high-RTG regime where $R^*$ exceeds the return of any trajectory in the dataset, the supervised regression term vanishes as $\mathcal{D}_{R^*}(s) = \emptyset$. The objective thus simplifies to the expectation of the indicator function:
\begin{align*}
L_{\mathrm{total}}(s,R^*) = \sum_{a' \in \mathcal{A}_{\mathrm{data}}(s)} p_{\mathcal{D}}(a' \mid s) \cdot \mathbb{I}\bigl( Q(s, a') > Q(s, \pi_{R^*}(s)) \bigr).
\end{align*}

Since the empirical probability $p_{\mathcal{D}}(a' \mid s)$ is strictly positive for all $a' \in \mathcal{A}_{\mathrm{data}}(s)$, the non-negative sum reaches its global minimum of $0$ if and only if:
\begin{align*}
Q(s, \pi_{R^*}(s)) \ge Q(s, a'), \quad \forall a' \in \mathcal{A}_{\mathrm{data}}(s).
\end{align*}
This directly implies $Q(s, \pi_{R^*}(s)) \ge \max_{a' \in \mathcal{A}_{\mathrm{data}}(s)} Q(s, a')$. Given that the dataset covers all state-action pairs, the global minimizer satisfies:
\begin{align*}
\pi_{R^*}^{\mathrm{Align}}(s) = \arg\max_{a \in \mathcal{A}_{\mathrm{data}}(s)} Q(s, a).
\end{align*}
This completes the proof.
\end{proof}

\begin{lemma}
Assume the dataset fully covers all state-action pairs, consider an idealized version of QT where the conditioning space is extended to $R \to \infty$. Let $\pi^{\mathrm{QT}}$ denote the policy that minimizes the QT actor loss.  
Then, for the maximal target return $R^*$, $\pi^{\mathrm{QT}}$ outputs an action
\begin{align*}
a^* = \arg\max_{a \in \mathcal{A}_{\mathrm{data}}(s)} Q(s,a),
\end{align*}
for each state $s$.
\label{QT-analysis}
\end{lemma}

\begin{proof}
For a fixed state $s$, the full QT actor loss is
\begin{align*}
L_{\mathrm{QT}}(s,R^*) =
\underbrace{\mathbb{E}_{a \sim \mathcal{D}_{R^*}(s)} \|\pi_{R^*}(s) - a\|^2}_{\text{supervised regression term}}
-
\underbrace{Q(s,\pi_{R^*}(s))}_{\text{Q-term}}.
\end{align*}

When the target return $R^*$ is larger than any trajectory in the dataset, the supervised regression term vanishes because $\mathcal{D}_{R^*}(s)$ is empty.  
The remaining loss reduces to
\begin{align*}
L_{\mathrm{QT}}(s,R^*) = - Q(s, \pi_{R^*}(s)),
\end{align*}
which is minimized by choosing
\begin{align*}
\pi_{R^*}(s) = \arg\max_{a \in \mathcal{A}_{\mathrm{data}}(s)} Q(s,a).
\end{align*}

Thus, for the maximal target return, $\pi^{\mathrm{QT}}$ outputs the action with the highest Q-value among all dataset actions, which concludes the proof.
\end{proof}

\begin{proof}[Proof of \cref{Equivalence}]
By combining the results of \cref{Q-align-analysis} and \cref{QT-analysis}, we observe that both policies, $\pi^{\mathrm{QT}}$ and $\pi^{\mathrm{Align}}$, satisfy the same optimality condition in the high-RTG limit. Specifically, since both objectives lead to the selection of the greedy action
\begin{equation}
a^* = \arg\max_{a \in \mathcal{A}_{\mathrm{data}}(s)} Q(s,a)
\end{equation}
for any given state $s$, the behavior of \methodname{} is theoretically equivalent to that of an idealized QT in this regime. This concludes the proof.
\end{proof}

\subsection{Proof of \cref{optimal}}
To simplify the analysis, we consider \textit{support-constrained optimality}, where all state-action pairs $(s,a)$ are restricted to the support of the offline dataset $\mathcal{D}$, and, for theoretical analysis, we assume exact critic updates.
When $\Delta \mathrm{RTG}=0$, the critic Bellman update evaluates the unperturbed RTG-conditioned policy within the dataset support, yielding a conservative, behavior-aligned value estimate. Consequently, the combination of the supervised learning loss and the unperturbed critic keeps the policy anchored near the behavior policy $\beta$. We then focus on the \textbf{high-RTG regime} where the perturbation $\Delta \mathrm{RTG}$ is sufficiently large to steer the policy toward high-value regions. 
Following IQL \citep{IQL}, we define the support-constrained optimal state-action value function $Q^*$ and the corresponding optimal policy $\pi^*$ as
\begin{align*}
    Q^*(s,a) &= r(s,a) + \gamma \mathbb{E}_{s'\sim P(\cdot|s,a)} \Big[ \max_{a'\in \mathcal{A} : \pi_\beta(a'|s')>0} Q^*(s',a') \Big], \\
    \pi^*(s) &= \arg\max_{a \in \mathcal{A} : \pi_\beta(a|s) > 0} Q^*(s,a),
\end{align*}
where $\mathcal{A}$ is the action space and $\pi_\beta$ is the behavior policy of the offline dataset $\mathcal{D}$.

To proceed, we first establish the following lemma:
\begin{lemma}
Let $Q_m$ be the action-value function of some policy $\pi$. 
Let $\pi_{m}$ be the policy obtained by minimizing the alignment actor loss on $Q_m$. 
Let $Q_{m+1}$ be the action-value function obtained by minimizing the TD loss \eqref{eq:Loss-Q} with $\Delta\text{RTG}=R^*$.

Then for every $(s,a)$ we have
\begin{align*}
Q_{m+1}(s,a) \ge Q_m(s,a).
\end{align*}

Moreover, if there exists some state $s$ where $\pi_m(s)$ strictly improves the one-step target,
then the inequality is strict for at least one state-action pair.
\label{improvement}
\end{lemma}

\begin{proof}[Proof sketch]
For notational simplicity, we use $\tilde{\pi}_m$ to denote the policy induced by
$\pi_m$ under a return perturbation $\Delta \mathrm{RTG}=R^*$.
We prove the claim by induction on the remaining horizon.

For the terminal timestep $H$, where no bootstrapping is applied, we have
\begin{align*}
Q_{m+1}(s_H,a_H) = Q_m(s_H,a_H) = r(s_H,a_H).
\end{align*}

Assume the statement holds for all timesteps $t+1,\dots,H$. 
Consider a state-action pair $(s_t,a_t)$ at timestep $t$:
\begin{align*}
Q_{m+1}(s_t,a_t) 
&= r(s_t,a_t) + \gamma \mathbb{E}_{s_{t+1} \sim P(\cdot \mid s,a)}
\Big[Q_{m+1}\big(s_{t+1}, \tilde{\pi}_m(s_{t+1})\big)\Big] \\
&\ge r(s_t,a_t) + \gamma  \mathbb{E}_{s_{t+1} \sim P(\cdot \mid s,a)}
\Big[Q_m\big(s_{t+1}, \tilde{\pi}_m(s_{t+1})\big) \Big]\\
&\ge r(s_t,a_t) + \gamma  \mathbb{E}_{s_{t+1} \sim P(\cdot \mid s,a)}
\Big[Q_m\big(s_{t+1}, {\pi}(s_{t+1})\big)\Big] \\
&= Q_m(s_t,a_t),
\end{align*}
where the first inequality follows from the induction hypothesis,
and the second inequality follows from \cref{Q-align-analysis}.
This completes the induction and proves $Q_{m+1} \ge Q_m$,
with strict inequality for at least one state-action pair if $\pi_m$
strictly improves the one-step target.
\end{proof}




\begin{proof}[Proof of \cref{optimal}]

By \cref{improvement}, the \methodname\ update ensures $\{Q_m\}$ is monotonically non-decreasing: $Q_{m+1} \ge Q_m$ on $\mathcal{D}$. Given bounded rewards and finite horizon, the sequence $\{Q_m\}$ converges pointwise to a limit $\bar{Q}$ by the Monotone Convergence Theorem.

We show $\bar{Q} = Q^*$ on $\mathcal{D}$ by contradiction. Suppose $\bar{Q}$ is not optimal; then there exists $(s,a) \in \mathcal{D}$ such that:
\begin{equation}\label{eq:non_optimal}
\bar{Q}(s,a) < r(s,a) + \gamma \mathbb{E}_{s' \sim P} \big[ \max_{a' \in \mathcal{D}(s')} \bar{Q}(s',a') \big].
\end{equation}
By \cref{Q-align-analysis}, under a sufficiently large perturbation $\delta$, the induced policy $\tilde{\pi}$ (derived from the converged policy $\pi_{\theta}$) selects greedy actions w.r.t. $\bar{Q}$: $\tilde{\pi}(s) = \arg\max_{a' \in \mathcal{D}(s)} \bar{Q}(s,a')$. Substituting this into \eqref{eq:non_optimal} yields:
\begin{equation}\label{eq:contradiction_target}
\bar{Q}(s,a) < r(s,a) + \gamma \mathbb{E}_{s' \sim P} \big[ \bar{Q}(s', \tilde{\pi}(s')) \big].
\end{equation}
However, upon convergence, $\bar{Q}$ must satisfy the Bellman equation corresponding to the induced policy $\tilde{\pi}$:
\begin{equation*}
\bar{Q}(s,a) = r(s,a) + \gamma \mathbb{E}_{s' \sim P} \big[ \bar{Q}(s', \tilde{\pi}(s')) \big],
\end{equation*}
which directly contradicts \eqref{eq:contradiction_target}. Therefore, $\bar{Q}$ satisfies the Bellman optimality equation on $\mathcal{D}$, implying $\bar{Q} = Q^*$. Consequently, the sequence $\{\tilde{\pi}_m\}$ converges to an optimal policy $\pi^*(s) \in \arg\max \bar{Q}(s,a)$ restricted to the dataset.
\end{proof}

\section{Experiments on \texttt{halfcheetah-vel}}
\label{Halfcheetah-vel-Transfer}
Transferring a model trained on the \texttt{halfcheetah} environment and deploying it directly on \texttt{halfcheetah-vel} is challenging, as the two environments differ in several important ways. We summarize the main challenges as follows:

\begin{itemize}
    \item \textbf{Different Reward Functions.}  
    Although both environments include a forward-velocity term and a control cost, their reward definitions differ significantly.  
    In \texttt{halfcheetah}, the forward reward is simply the current velocity.  
    In contrast, \texttt{halfcheetah-vel} defines the forward reward as  
    \begin{align*}
        -\lvert v_t - v_{\text{target}} \rvert ,
    \end{align*}
    where \(v_t\) is the current velocity and \(v_{\text{target}}\) is the target velocity.
    
    \item \textbf{Different Horizons.}  
    \texttt{halfcheetah} uses a horizon of 1000 steps, whereas \texttt{halfcheetah-vel} terminates after 200 steps.  
    This discrepancy affects the RTG interpretation learned during training.
    
    \item \textbf{Different Observation Dimensions.}  
    Compared to \texttt{halfcheetah}, the \texttt{halfcheetah-vel} environment further includes a 3-dimensional absolute position in its observations.
\end{itemize}

For fair comparison on \texttt{halfcheetah-vel}, we directly deploy the model trained on \texttt{halfcheetah}, but apply several modifications to the inputs:

\begin{itemize}

    \item \textbf{Reward Mismatch.}  
    To maintain consistency with the training setup, we re-compute the reward using the current forward velocity, updating the RTG as:
    \begin{align*}
        \bar{r}_t = v_t + c_t, \qquad
        \mathrm{RTG}_{t+1} = \mathrm{RTG}_t - \bar{r}_t ,
    \end{align*}
    where \(c_t\) is the control cost from \texttt{halfcheetah-vel}. Note that although this modification is applied when updating the RTG, we report results using the original reward definition of \texttt{halfcheetah-vel} for all transferred methods.

    \item \textbf{Horizon Mismatch.}  
    We continue to construct the RTG token using the original horizon of 1000 (i.e., \(1000 \times v_{\text{target}}\)), while the environment naturally truncates at 200 steps.

    \item \textbf{Observation Mismatch.}  
    We simply discard the extra 3 dimensions in \texttt{halfcheetah-vel}, and rely on the RTG token to convey velocity-related information.
\end{itemize}

Even after applying these adjustments, selecting an appropriate initial RTG for deployment on \texttt{HalfCheetah-vel} remains challenging. The main difficulty lies in estimating the cumulative control cost required to compute the RTG corresponding to a given target velocity: since the control cost depends on the induced trajectory, the desired RTG cannot be determined analytically and must be calibrated empirically at deployment time.
Moreover, since the model is trained with a 1000-step horizon, it is allowed to defer RTG consumption toward later timesteps; empirically, we observe that the learned policy can exhibit such behavior (as illustrated in \cref{fig:diff-traj}), which becomes infeasible under the 200-step evaluation horizon.

To address this mismatch, we define a base RTG as $1000 \times v_{\text{target}}$ and empirically calibrate the RTG values. Specifically, we sample 40 RTG candidates around this base value for the maximum and minimum target velocities to identify the corresponding optimal RTGs, denoted as $\text{RTG}_{\max}$ and $\text{RTG}_{\min}$.
For intermediate target velocities $v$, we assume a smooth and approximately monotonic relationship between RTG and the resulting velocity, and linearly interpolate the RTG as
\begin{align*}
\text{RTG}(v) = \text{RTG}_{\min} +
\frac{v-v_{\min}}{v_{\max}-v_{\min}}
\left(\text{RTG}_{\max}-\text{RTG}_{\min}\right).
\end{align*}
All methods are evaluated using this procedure independently to ensure a fair comparison.

The effectiveness of linear interpolation demonstrates both the controllability and alignment of \methodname{}.
In standard CSMs, the mapping from RTG to realized behavior is often inconsistent, leading to poor alignment with the target velocity in this setting.
In contrast, \methodname{} learns a smooth and approximately linear relationship between RTG and behavior.
This allows us to anchor the policy at only the two endpoint velocities and reliably interpolate for intermediate targets, handling the environment shift without the need for additional RTG tuning, a strategy that fails for standard CSMs. 

\rev{
To understand the sensitivity of \methodname{} to the number of RTG candidates, we conduct experiments with different calibration budgets and evaluate on \texttt{halfCheetah-vel}. As shown in \cref{tab:calibration_samples}, zero-shot transfer performance degrades gradually rather than collapsing as the number of calibration samples is reduced. Even with substantially fewer calibration samples, \methodname{} remains much stronger than the other CSM-based baselines evaluated with 40 samples, as shown in \cref{tab:halfcheetah-vel}.}

\begin{table}[t]
\centering
\revtab{
\caption{Effect of the number of calibration samples on zero-shot transfer performance.}
\label{tab:calibration_samples}
\begin{tabular}{lcccc}
\toprule
Calibration Samples & 5 & 10 & 20 & 40 \\
\midrule
Q-Align DT & -158.5 & -155.7 & -152.3 & -142.3 \\
\bottomrule
\end{tabular}
}
\end{table}

\section{Implementation Details}



\subsection{Model Architecture}
\label{app:model_structure}

\textbf{CSM Architecture.} We adopt a Transformer-based policy following the sequence-conditioned structure in \citet{DT}. To enhance the local temporal consistency of control signals, we insert a lightweight 1D convolution layer after the projection matrices $W_k, W_q, W_v$. Specifically, for each layer, we compute:
\begin{align*}
k_\ell = \mathrm{Conv}(W_k x_\ell), \quad q_\ell = \mathrm{Conv}(W_q x_\ell), \quad v_\ell = \mathrm{Conv}(W_v x_\ell).
\end{align*}
The convolution window size is set to $w=6$, enabling each output to incorporate information from a local temporal neighborhood. 
To preserve causality, we apply \emph{causal left padding} by padding the input sequence with $w-1=5$ zeros on the left, 
ensuring that the output at timestep $t$ depends only on inputs at or before $t$.
This prevents future information leakage and maintains the temporal causality required for sequence modeling. As demonstrated in our ablation studies, this modification does not dictate the peak performance.

\textbf{Prediction Heads.} Following the design in \citet{QT}, the model performs prediction sequentially at each timestep $i$. Specifically, the hidden representation of the RTG token $\text{rtg}_i$ is mapped through a linear head to reconstruct the current state $\hat{s}_i$, and the representation of the state token $s_i$ is mapped through a separate head to predict the corresponding action $\hat{a}_i$. These architectural details are summarized in \cref{tab:model-arch}.

\textbf{$Q$-Function.} For the $Q$-function, we adopt a 3-layer MLP with a hidden dimension of 256.

\subsection{Q Pretraining}
\label{Pretrain}
For stability, we pretrain the $Q$-function before training the CSM. 
To isolate the effect of the proposed alignment method, we adopt a simple Double $Q$-learning update \citep{doubleq} in all experiments, except for the AntMaze environments, where IQL \citep{IQL} is used. 

Specifically, for Q networks $Q_{\psi_1}, Q_{\psi_2}$ and target networks $Q_{\psi_1'}, Q_{\psi_2'}$, the pretraining loss is
\[
L_{\text{pretrain}} = \sum_{i=t-k+1}^t \sum_{m=1}^{2} \Bigl(Q_{\psi_m}(s_i,a_i) - y_i\Bigr)^2,
\quad
y_i = r_i + \gamma \min_{m=1,2} Q_{\psi_m'}(s_{i+1}, a_{i+1}),
\]
where $\gamma=0.99$. While RTG typically represents an undiscounted return ($\gamma = 1$) and we follow this standard practice for RTG conditioning, we employ a discounted $Q$-function ($\gamma = 0.99$) solely as a guidance signal during pretraining and alignment. This design choice is empirically robust and consistent with prior successful architectures \citep{QT, QCS, RLvitamin}.
Importantly, since our alignment loss relies exclusively on the \textbf{relative ordering} rather than absolute magnitude, the scale discrepancy introduced by the discount factor does not affect the directional gradient. As long as the discounted $Q$-function preserves monotonic preference over returns, the structural regularization remains theoretically sound and avoids optimization inconsistency.

For AntMaze, standard Bellman updates often fail to propagate values across long horizons in offline settings, leading to vanishing gradients and uninformative value estimates. 
We therefore use IQL for pretraining, with an expectile of 0.8.

\begin{table}[htbp]
\centering
\caption{Architecture of \methodname}
\begin{tabular}{lc}
\toprule
Parameter & Value \\
\midrule
Number of layers & 4 \\
Attention heads & 4 \\
Conv window size & 6 \\
Embedding dimension & 256 \\
Batch size & 256 \\
Dropout & 0.1 \\
Learning rate & 3e-4 \\
Target update rate& 0.005\\
Nonlinearity & GELU \citep{GELU} \\
Optimizer & Adam \citep{Adam} \\
\bottomrule
\end{tabular}
\label{tab:model-arch}
\end{table}

\subsection{Baseline Details}
We evaluate \methodname{} against a comprehensive set of baselines, including \textit{traditional offline RL methods}: IQL~\citep{IQL}, TD3+BC~\citep{TD3+BC}, and CQL~\citep{CQL}; and \textit{sequence modeling approaches}: DT~\citep{DT}, DC~\citep{DC}, RVS~\citep{rvs}, CGDT~\citep{CGDT}, LSDT~\citep{LSDT}, DM~\citep{DM}, RADT~\citep{RADT}, QT~\citep{QT}, and QCS~\citep{QCS}. 

For most tasks, we report the performance directly from their original publications. However, for the AntMaze environment, some baselines such as QT~\citep{QT} originally reported results on the \texttt{v0} version. To ensure a fair comparison on the current standard benchmark, we re-evaluate these methods on the \texttt{v2} environment using their official codebases. For the specific evaluation of alignment properties, we re-train and evaluate DT, DC, RADT, QT, and QCS using their respective official implementations to ensure consistency in our experimental setup.

\subsection{Hardware and Environment Configuration}

All experiments were conducted on a NVIDIA L40S GPU with 48GB of HBM2 memory.
The CPU is an Intel(R) Xeon(R) Gold 6430.

For the software environment, we used:
\begin{itemize}
    \item \texttt{d4rl} 1.1, \texttt{gym} 0.18, \texttt{mujoco} 2.0.2
    \item \texttt{PyTorch} 2.4, CUDA 12.1
    \item Operating system: Red Hat Enterprise Linux 9.6
    \item Environment managed with \texttt{conda} and Python 3.8
\end{itemize}
\subsection{Hyperparameters}
\label{Hyperparameters}
We report the hyperparameters of \methodname{} for the selected tasks in
\cref{tab:hyper}. Below, we describe the key components: the noise scale
$\sigma_e$, the weight of the $Q$-alignment loss $\lambda_e$, and the RTG
perturbation magnitude $\Delta\text{RTG}$. Following the standard protocol
of Decision Transformer \citep{DT}, we scale the target RTG by a factor of
$1/1000$ for all D4RL Gym environments. Accordingly, the perturbation scale
$\Delta\text{RTG}$ and the noise magnitude $\sigma_e$ reported in this paper
are defined with respect to these normalized RTG values.

The choice of $\sigma_e$ is primarily influenced by the reward scale of each environment.  
The coefficient $\lambda_e$ varies across datasets because it depends on both the dataset distribution and the stability of the environment during $Q$-function updates.  

The perturbation value $\Delta \text{RTG}$ is introduced during the $Q$-function update so that the learned critic reflects the value of relatively high-reward policies within the policy family of \methodname{}. 
Ideally, we would like it to approximate the value of the best policy in this family, i.e., by letting $\Delta \text{RTG} \to \infty$. 
However, in practice this may lead to out-of-distribution behavior and destabilizes training. 
Therefore, we choose $\Delta \text{RTG}$ according to the stability characteristics of each environment.

\begin{table}[htbp]
\centering
\caption{Hyperparameters of \methodname{} across different environments. $\sigma_e$ controls the scale of RTG perturbation noise; $\lambda_e$ is the weight of the $Q$-alignment loss; $\Delta \text{RTG}$ specifies the RTG perturbation used for $Q$-function updates. }
\begin{tabular}{lccc}
\toprule
Environment / Dataset & $\sigma_e$ &$\lambda_e$ &$\Delta \text{RTG}$  \\
\midrule
walker2d-medium-replay-v2 & 10 & 0.3 & 2 \\
halfcheetah-medium-replay-v2 & 15 & 5 & 5 \\
hopper-medium-replay-v2 & 10 & 0.3 & 0.5 \\
walker2d-medium-v2 & 10 & 0.3 & 1 \\
halfcheetah-medium-v2 & 15 & 5 & 10 \\
hopper-medium-v2 & 10 & 0.3 & 1 \\
walker2d-medium-expert-v2 & 10 & 0.3 & 10 \\
halfcheetah-medium-expert-v2 & 15 & 1 & 10 \\
hopper-medium-expert-v2 & 10 & 0.3 & 5 \\
antmaze-umaze-v2 & 5 & 100 & 1 \\
antmaze-umaze-diverse-v2 & 5 & 100 & 1 \\
antmaze-medium-play-v2 & 5 & 100 & 1 \\
antmaze-medium-diverse-v2 & 5 & 100  & 1\\
\bottomrule
\end{tabular}
\label{tab:hyper}
\end{table}

Specifically, when sampling RTG perturbations $\delta$ in Antmaze environement, we employ a half-normal distribution, $\delta = |\epsilon|$, $\epsilon \sim \mathcal{N}(0, \sigma_e^2)$. 
This ensures that the perturbations are positive, providing directional guidance toward the goal in sparse-reward navigation tasks, while negative perturbations, which are uninformative in this setting, are avoided.

We do not perform extensive hyperparameter sweeps, as our goal is to evaluate the robustness and controllability of the proposed method rather than optimize peak performance.
Accordingly, we largely share hyperparameters across environments within the same domain, making only minor adjustments when necessary for training stability.
Despite this minimal tuning, \methodname{} consistently outperforms baselines, indicating that the observed gains arise from improved alignment rather than hyperparameter optimization.
For AntMaze, we fix a larger $\lambda_e$ across all variants to accommodate the sparse-reward structure and different reward scale.

\subsection{Alignment Calculation}
For evaluating alignment performance, we sweep the target RTG from the minimum to the maximum return in each dataset with a step size of 100. For each RTG, we run 30 rollouts and compute the average D4RL score. The root mean squared deviation between the achieved scores and the target scores is used as the alignment metric. The minimum and maximum returns of each dataset are listed in \cref{tab:maximum-reward}.

This evaluation is similar to prior work~\citep{RADT}; however, unlike RADT which evaluates only 7 RTG values per dataset, our method considers at least 32 distinct RTG values (in Hopper), providing a finer-grained assessment of policy alignment across the entire return spectrum.

\begin{table}[htbp]
\centering
\caption{Maximum and Minimum return of each dataset}
\begin{tabular}{l c c}
\toprule
Dataset & Maximum return & Minimum return \\
\midrule
Hopper-v2 & 3234.3  & -20.3\\
Walker2d-v2  & 4592.3  &  1.6 \\
Halfcheetah-v2  & 12135.0 & -280.2  \\
\bottomrule
\end{tabular}
\label{tab:maximum-reward}
\end{table}

\section{More Experiment Results}
\subsection{More Ablation Experiments}
\label{more-aba}
In this section, we provide a detailed ablation analysis of the core components in \methodname{}.


\paragraph{Effect of Fixing the $Q$-function.}
We conduct an ablation in which the $Q$-function is fixed after pretraining and is no longer
updated jointly with the actor. As shown in \cref{tab:ablation-fixq}, fixing the $Q$-function
leads to a clear degradation in the best achievable performance, as the static critic fails to provide informative signals to improve the actor beyond the behavior policy. Furthermore, alignment error increases accordingly, indicating that co-training is essential for maintaining reward-sensitive behavior.

Notably, most of the degradation occurs under low-RTG conditions. During inference, when
conditioned on a relatively low target RTG, the model trained with a fixed $Q$-function often
collapses (e.g., falls) in the early timesteps. This early failure causes the realized return
to approach zero, falling far below the intended RTG target and resulting in a breakdown of
alignment (see \cref{app:alignment-analysis} for further results).

\begin{table}[ht]
\caption{Effect of fixing the $Q$-function. RMSE ($\downarrow$) denotes the alignment error,
while Perf. ($\uparrow$) denotes the overall performance (D4RL normalized score).
Results are averaged over three random seeds.}
\centering
\setlength{\tabcolsep}{3pt}
\begin{tabular}{lcccc}
\toprule
Dataset &
\multicolumn{2}{c}{Align. RMSE $\downarrow$} &
\multicolumn{2}{c}{Perf. $\uparrow$} \\
\cmidrule(lr){2-3} \cmidrule(lr){4-5}
 & \methodname & Fix $Q$& \methodname & Fix $Q$\\
\midrule
Medium-Replay & 11.31 & 24.45 & 86.45 & 80.32 \\
Medium  & 9.19  & 28.04 & 87.70 & 79.72 \\
Medium-Expert & 6.80  & 16.90 & 111.37 & 106.83 \\
\bottomrule
\end{tabular}
\label{tab:ablation-fixq}
\end{table}




\paragraph{Asymmetric vs.\ Symmetric Indicator Functions.}
We investigate the necessity of the \textbf{asymmetric indicator function} by considering a symmetric variant:
\begin{equation*}
\mathbb{I}_{\text{sym}} = 
\begin{cases} 
1, & \text{if } \text{sgn}(\delta) \big( Q_{\psi}(s_i, \hat{a}_i^\delta) -Q^{\perp}_{\psi}(s_i, \hat{a}_i) \big) < 0, \\
-1, & \text{otherwise}.
\end{cases}
\end{equation*}
Unlike our original formulation, which only penalizes constraint violations, this variant explicitly encourages the actor to further maximize the value gap even when the ranking is already correct. 

As shown in \cref{tab:loss-aba}, this ``active alignment'' approach does not generally improve performance. On the contrary, it significantly increases training instability and even leads to behavioral collapse in environments like Hopper and Walker2d. This suggests that the primary role of the alignment loss should be \textit{correcting} misaligned rankings rather than continuously pushing actions beyond the critic's reliable regions, which may introduce detrimental gradient noise.

\paragraph{Sensitivity to the Alignment Loss Form.}
We further study the sensitivity of \methodname{} to the functional form of the alignment loss.
Specifically, we replace the default linear penalty with a Squared Penalty,
\begin{align*}
\big( Q_{\psi}(s_i, \hat{a}_i^\delta) -  Q^{\perp}_{\psi}(s_i, \hat{a}_i)\big)^2,
\end{align*}
which imposes a quadratic cost on directional inconsistencies.

As shown in \cref{tab:loss-aba}, this variant yields comparable performance across all tasks, with only minor deviations from the default formulation.
Empirically, introducing higher-order penalties does not yield consistent improvements, which leads us to adopt the simplest formulation, viewable as a first-order approximation of the value difference under RTG perturbations.

\begin{table}[htbp]
\centering
\caption{Ablation of the alignment loss structure. We compare the default \methodname{} with its symmetric and squared (L2) variants. }
\begin{tabular}{l ccc}
\toprule
Dataset & \methodname{} (Ours) & Symmetric Penalty &Square Penalty\\
\midrule
Hopper-medium & 102.5 & 42.9 (collapse) &98.7\\
Walker2d-medium & 95.1 & 14.3 (collapse)&93.4 \\
Halfcheetah-medium & 65.5 & 63.4 &68.7\\
\bottomrule
\end{tabular}
\label{tab:loss-aba}
\end{table}

\paragraph{Sensitivity to RTG $\sigma_e$.}
We evaluate the sensitivity of \methodname{} to the perturbation magnitude $\sigma_e$, which is defined relative to the normalized RTG range. 
\Cref{tab:sigma-aba} reports the performance on the \texttt{halfcheetah-medium} task across a wide range of $\sigma_e$ values.

The results indicate that \methodname{} is highly robust to the choice of $\sigma_e$ within a reasonable range (from 1 to 50). 
Even a large perturbation ($\sigma_e = 50$) does not lead to performance collapse, suggesting that the directional signal remains informative. 
By contrast, an excessively small perturbation (e.g., $\sigma_e = 0.5$) limits the actor's ability to explore the value landscape, resulting in suboptimal alignment.

\begin{table}[htbp]
\centering
\caption{Sensitivity analysis of the perturbation scale $\sigma_e$ on \texttt{halfcheetah-medium}. The noise is applied to the normalized RTG (scaled by $1/1000$). \methodname{} exhibits strong robustness across an order of magnitude of noise levels.}
\vspace{2mm}
\begin{tabular}{l cccccc}
\toprule
$\sigma_e$ & 0.5 & 1.0 & 5.0 & 10.0 & 15.0 & 50.0 \\
\midrule
D4RL Score & 43.3 & \textbf{66.7} & 65.7 & 65.3 & 65.5 & 64.7 \\
\bottomrule
\end{tabular}
\label{tab:sigma-aba}
\end{table}
\paragraph{Sensitivity to RTG $\lambda_e$.}
We further evaluate the sensitivity of \methodname{} to the alignment loss weight $\lambda_e$. As shown in \cref{tab:lambda-aba}, we report the performance on \texttt{halfcheetah-medium} across a wide spectrum of weights. 

The results reveal a clear phase transition: at $\lambda_e = 0$, where the directional guidance is absent, the model achieves a score of only $42.84$, consistent with standard supervised learning baselines. However, with even a small weight ($\lambda_e = 1.0$), the performance surges to $61.7$. Crucially, for $\lambda_e \ge 3.0$, the performance enters a stable plateau, with scores remaining consistently above $64.0$ regardless of the specific weight chosen. While the peak performance reaches $68.2$ at $\lambda_e = 100.0$, the marginal gains across the $3.0$ to $100.0$ range suggest that \methodname{} is highly robust to this hyperparameter. This insensitivity allows for a ``plug-and-play'' deployment without the need for exhaustive, task-specific fine-tuning.

\paragraph{Sensitivity to $Q$ Function Accuracy} 

We investigate the sensitivity of \methodname{} to inaccuracies in the learned $Q$ function. Specifically, we perturb the target critic $Q_{\psi'}$ by injecting Gaussian noise $\epsilon \sim \mathcal{N}(0, \sigma_n^2)$ to simulate potential estimation errors. 

We conduct experiments on \texttt{halfcheetah-medium} and \texttt{walker2d-medium} with noise scales $\sigma_n \in \{1, 10, 50\}$, and report the results in Table~\ref{tab:qacc-aba}. Notably, \methodname{} exhibits remarkable resilience, maintaining stable performance even under substantial perturbations. This indicates that the method is  insensitive to the exact accuracy of the $Q$-function. We attribute this robustness to the alignment objective, which emphasizes the \emph{relative ordering} and \emph{directional consistency} of $Q$-values rather than their absolute magnitudes. As long as the critic preserves coarse rank-ordering of actions, the alignment loss provides sufficient guidance for effective policy extraction, making the learning process tolerant to estimation noise.

\begin{table}[htbp]
\centering
\caption{Sensitivity analysis of the value function. We report the normalized D4RL scores under varying noise scales $\sigma_n$.} 
\vspace{2mm}
\begin{tabular}{l cccc}
\toprule
$\sigma_n$ & 0 & 1.0 &  10 & 50.0  \\
\midrule
halfcheetah-medium& 65.5 & 65.1 & 64.4 & 63.4 \\
walker2d-medium&95.1&91.1&91.8&90.7\\
\bottomrule
\end{tabular}
\label{tab:qacc-aba}
\end{table}

\begin{table}[htbp]
\centering
\caption{Sensitivity analysis of the perturbation scale $\lambda_e$ on \texttt{halfcheetah-medium}. }
\vspace{2mm}
\begin{tabular}{l cccccc}
\toprule
$\lambda_e$ & 0 & 1.0 & 3.0 & 5.0 & 10.0 & 100.0 \\
\midrule
D4RL Score & 42.84 & 61.7 & 64.42 & 65.5 & 65.3 & \textbf{68.2} \\
\bottomrule
\end{tabular}
\label{tab:lambda-aba}
\end{table}

\subsection{Extended Generalization on \texttt{HalfCheetah-Vel}}

While the standard \texttt{halfcheetah-vel} task focuses on target velocities within the $[0,3]$ range, our model leverages the diverse transitions in the \texttt{medium-expert} dataset to acquire a behavioral repertoire capable of reaching velocities exceeding 10. To fully explore the limits of its controllability, we evaluate \methodname{} on extended target velocities that far exceed those typically used in Meta-RL benchmarks. 

We follow the evaluation protocol in \cref{Halfcheetah-vel-Transfer}, extending the episode horizon to 1000 steps to allow sufficient time for the agent to accelerate to these higher targets. Since the original \texttt{halfcheetah-vel} dataset is restricted to low-velocity samples, we use CSMs trained on the broader \texttt{halfcheetah-medium-expert} dataset as baselines. As reported in \cref{tab:halfcheetah-vel-rtg-high}, \methodname{} demonstrates a remarkably wide effective range of controllability, accurately tracking high-velocity targets even when trained under a singular and straightforward reward objective (i.e., where reward simply scales with velocity).

\begin{table}[htbp]
\centering
\caption{
\textbf{Performance under high target velocities on \texttt{HalfCheetah-vel}.}
Maximum episode return is reported for each target velocity
(horizon = 1000).
}
\label{tab:halfcheetah-vel-rtg-high}
\vspace{2mm}

\begin{tabular}{c r r r r r}
\toprule
Target Vel. & DC & DT & RADT & QT & \methodname{} \\
\midrule
4.0  & -1749.01 & -1744.94 & -1709.06 & -1899.36 & -1289.17 \\
5.0  & -1332.75 & -1573.92 & -2452.64 & -2716.32 & -1246.46  \\
6.0  & -1061.06 & -1094.22 & -1085.82 & -2247.88 & -1203.69 \\
7.0  & -3666.29 & -3627.29 & -1026.00 & -3967.86 & -1441.51 \\
8.0  & -3519.51 & -4101.90 & -2100.74 & -4781.54 & -2366.28 \\
9.0  & -2850.36 & -3249.43 & -3888.88 & -4188.40 & -2422.39 \\
10.0 & -2162.07 & -2412.16 & -3261.80 & -3528.83 & -1493.27 \\
\midrule
\textbf{Mean} 
     & -2334.44 & -2543.41 & -2217.85 & -3332.88 & \textbf{-1637.54} \\
\bottomrule
\end{tabular}
\end{table}


\subsection{Behavioral Analysis on \texttt{Ant}}
\label{Ant-Analysis}
We also evaluate \methodname{} on \texttt{ant-medium} to examine whether the agent can modulate its behavior under different target RTGs. As shown in \cref{fig:ant-behavior-analysis}, the agent demonstrates similar trends to \texttt{halfcheetah}: higher target RTGs induce faster locomotion, while lower RTGs lead to more conservative movement. These results suggest that, even when trained only on scalar returns, the actor can learn distinct behaviors that are controllable via the RTG tokens, indicating that the structured mapping from RTG to actions generalizes across different environments.

\begin{figure}[ht]
    \centering
    \begin{subfigure}[b]{\linewidth}
        \centering
        \includegraphics[height=4cm,keepaspectratio]{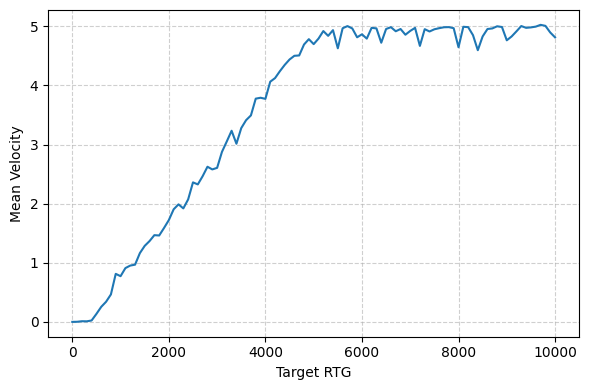}%
        \hspace{0.02\linewidth}%
        \includegraphics[height=4cm,keepaspectratio]{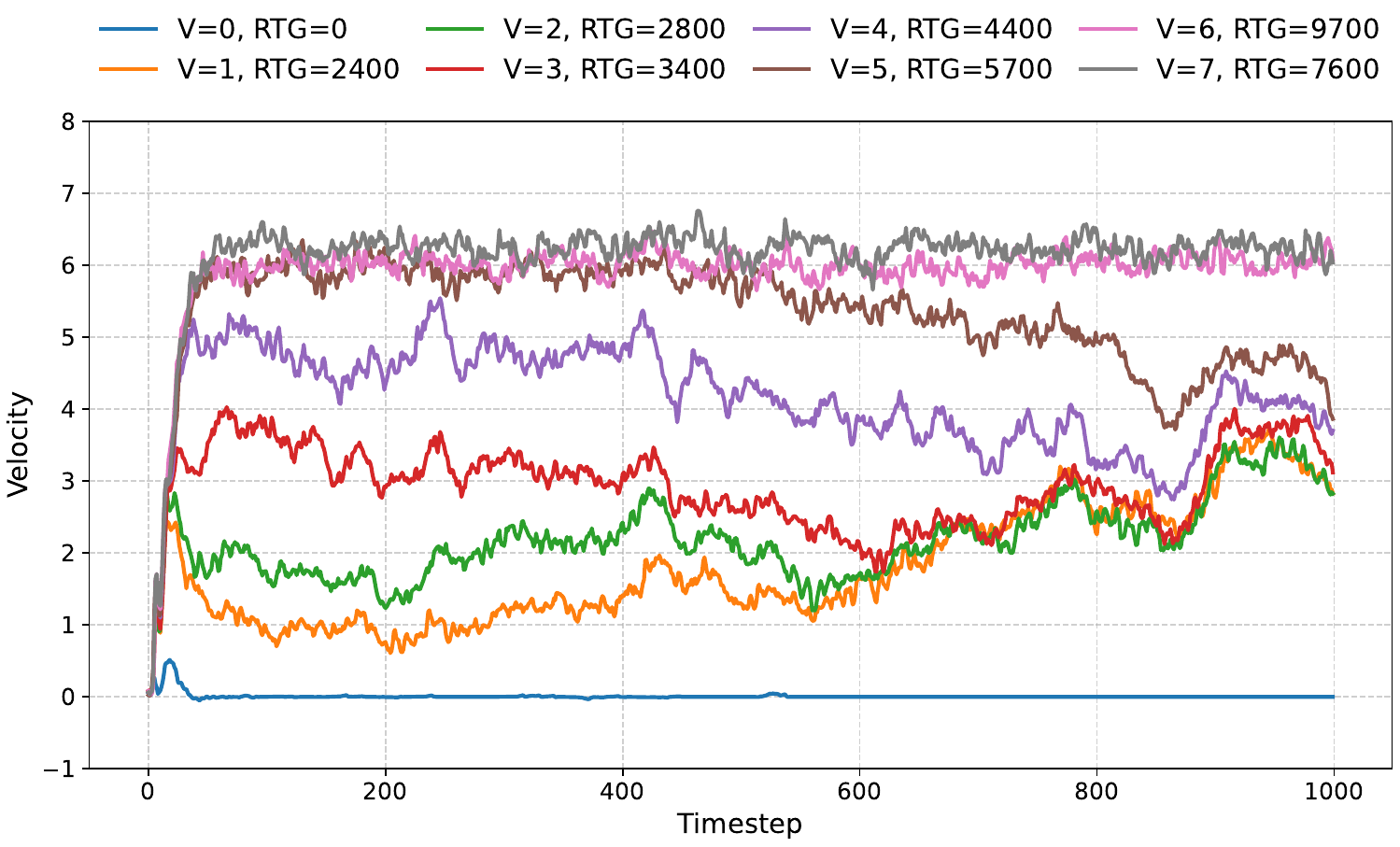}%
    \end{subfigure}

    \caption{Behavioral modulation of the agent on \texttt{ant-medium} under varying target RTGs. (Left) The relationship between target RTG and realized mean velocity; (Right) Velocity trajectories across different RTG levels averaged over 30 rollouts.}
    \label{fig:ant-behavior-analysis}
\end{figure}

\subsection{Additional Experiments}
\rev{We first show the last iteration performance of DT, QT and \methodname{} in \cref{tab:gym_results_last}}

\begin{table}[t]
\centering
\revtab{
\caption{ Last-iteration performance comparison on D4RL Gym tasks. }
\label{tab:gym_results_last}
\begin{tabular}{lccc}
\toprule
Task & DT & QT & Q-Align DT  \\
\midrule
HalfCheetah-medium  & 42.20  & 50.5  & $64.8 \pm 0.7$ \\
HalfCheetah-medium-replay & 38.91  & 48.4  & $52.3 \pm 1.5$ \\
HalfCheetah-medium-expert & 91.55  & 94.0  & $96.6 \pm 1.7$ \\
Hopper-medium  & 65.10  & 63.2  & $98.7 \pm 3.4$ \\
Hopper-medium-replay & 81.77  & 99.4  & $99.9 \pm 1.3$ \\
Hopper-medium-expert & 110.44 & 108.2 & $112.5 \pm 1.4$ \\
Walker2d-medium  & 67.63  & 86.5  & $92.5 \pm 2.9$ \\
Walker2d-medium-replay & 59.86  & 88.5  & $100.8 \pm 1.8$ \\
Walker2d-medium-expert & 107.11 & 110.5 & $115.8 \pm 1.9$ \\
\midrule
Sum   & 664.6  & 749.2 & 828.8 \\
\bottomrule
\end{tabular}
 }
\end{table}

\rev{We report a detailed ablation of the architectural change and the alignment objective in \cref{tab:arch-full-res}. To make the attribution clearer, we
compare vanilla DT, DT with 1D convolution, DT with the alignment loss, and Q-Align DT in terms of alignment MSE. As shown in \cref{tab:arch-full-res}, 1D convolution alone slightly reduces the alignment MSE, while the alignment loss provides the dominant improvement. Combining both components achieves the lowest alignment MSE, indicating that the main gain comes from the proposed alignment objective rather than the architectural modification alone.}

\begin{table}[t]
\centering
\revtab{
\caption{Ablation study on the effect of 1D convolution and the alignment loss, measured by alignment MSE. Lower is better.}
\label{tab:arch-full-res}
\begin{tabular}{lcccc}
\toprule
Dataset & DT & DT + 1DConv & DT + Align Loss & Q-Align DT \\
\midrule
Medium-Replay & 20.32 & 18.07 & 13.37 & 11.31 \\
Medium        & 25.33 & 20.32 & 14.22 & 9.19  \\
Medium-Expert & 18.53 & 19.43 & 10.70 & 6.80  \\
\midrule
Mean          & 21.39 & 19.27 & 12.76 & 9.12  \\
\bottomrule
\end{tabular}
}
\end{table}

\rev{We further report the training cost of DT, QT, and \methodname{} in
\cref{tab:training_time}, where each method is trained for 100 epochs with
1000 steps per epoch. QT and \methodname{} are slower than DT because both use
an additional Q-function. Compared with QT, \methodname{} only adds extra
forward computation for perturbed-action generation, without introducing
additional backward-pass overhead. As a result, its practical training cost
remains comparable to QT.}

\begin{table}[tbh]
\centering
\revtab{
\caption{Wall-clock training time under the same setting of 100 epochs and 1000 steps per epoch.}
\label{tab:training_time}
\begin{tabular}{lccc}
\toprule
Method & DT & QT & \methodname{} \\
\midrule
Training Time & $\sim$4 h 36 min & $\sim$6 h 7 min & $\sim$6 h 22 min \\
\bottomrule
\end{tabular}
}
\end{table}
\begin{table}[tbh]
\centering
\revtab{
\caption{Spearman correlation between predicted $Q$-values and realized returns across environments.}
\label{tab:spearman_corr}
\begin{tabular}{lc}
\toprule
Environment & Spearman Corr. \\
\midrule
Hopper-medium          & 0.94 \\
AntMaze-umaze          & 0.36 \\
AntMaze-umaze-diverse  & 0.26 \\
\bottomrule
\end{tabular}
}
\end{table}
\rev{To better understand the different behavior of \methodname{} on Gym and
AntMaze tasks, we further compute the Spearman correlation between predicted
$Q$-values and realized returns, as shown in \cref{tab:spearman_corr}.
The correlation is much lower on AntMaze than on dense-reward Gym tasks (e.g., 0.94 on Hopper-medium vs. 0.26 on AntMaze-umaze-diverse). This suggests that the $Q$-guided alignment signal is substantially weaker in sparse-reward environments, especially on the diverse split. We view this phenomenon as a limitation of \methodname{} in sparse-reward settings and leave further
improvement to future work.}

\subsection{Further Analysis on Alignment Stability}
\label{app:alignment-analysis}

We first present the complete alignment curves for \methodname{} across all nine Gym tasks in \cref{fig:gym_tasks}, demonstrating consistently robust alignment. 

Beyond overall performance, we examine scenarios in which alignment degrades, 
specifically under a fixed critic or when the RTG offset is near zero. 
As shown in \cref{fig:No_update_Q}, when the $Q$-function is fixed, the agent’s performance 
under low-RTG targets effectively collapses to zero. A stable response only emerges 
once the target RTG exceeds a certain threshold, indicating a sharp transition in behavior. 
In \cref{fig:Vel_No_update_Q}, we further report the agent’s velocity at each timestep, 
which shows that the misalignment in the low-RTG regime is caused by early timesteps collapse.


We further examine the impact of RTG offsets in \cref{fig:Diff_offset}. Notably, even with synchronized training, a near-zero offset replicates the failure mode of the fixed-critic baseline. This is because a sufficient $\Delta\mathrm{RTG}$ is essential to ensure that the $Q$-function evaluates the better-performing policies within the evolving policy family. Without this positive offset, the alignment objective is guided by value estimates of mediocre or failing behaviors. Consequently, the policy lacks the structural guidance needed to remain stable under low-RTG targets, leading to the observed performance collapse where realized returns drop to zero.

\begin{figure}[tbh]
    \centering
    \begin{subfigure}[b]{0.32\textwidth}
        \includegraphics[width=\textwidth]{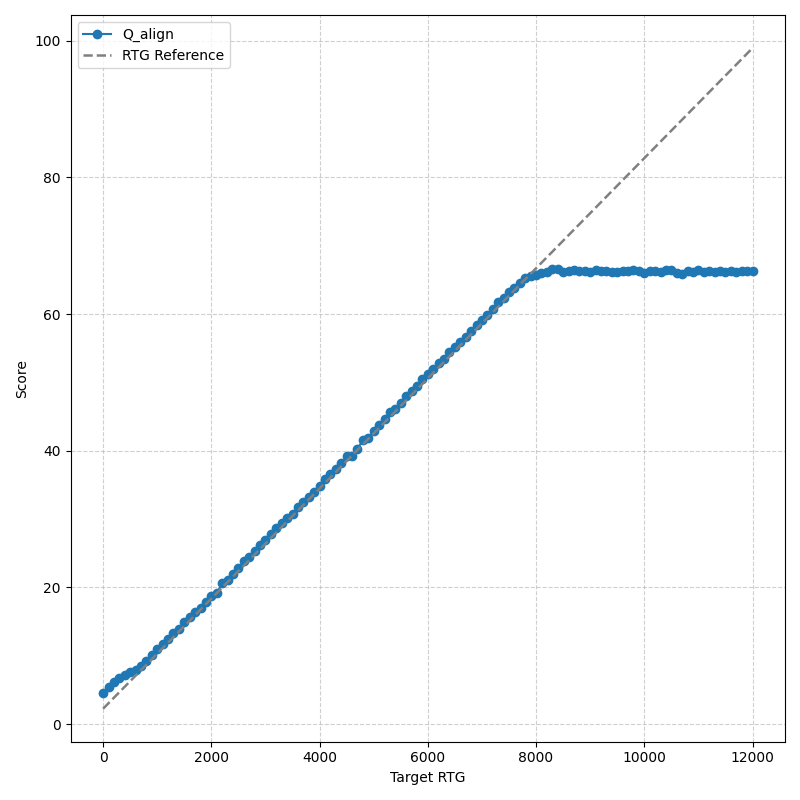}
        \caption{HalfCheetah Medium}
    \end{subfigure}
    \begin{subfigure}[b]{0.32\textwidth}
        \includegraphics[width=\textwidth]{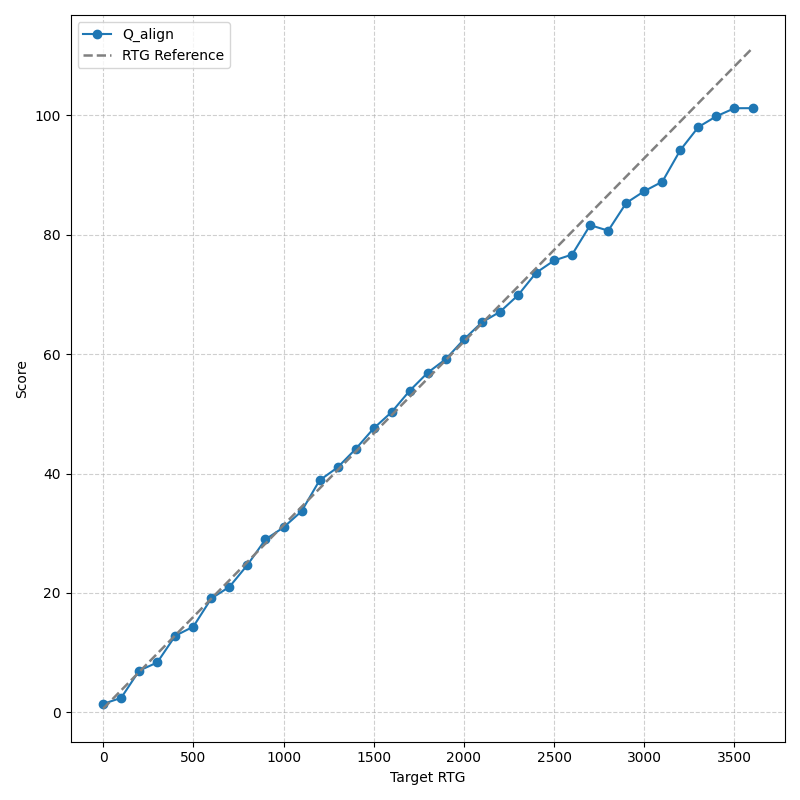}
        \caption{Hopper Medium}
    \end{subfigure}
    \begin{subfigure}[b]{0.32\textwidth}
        \includegraphics[width=\textwidth]{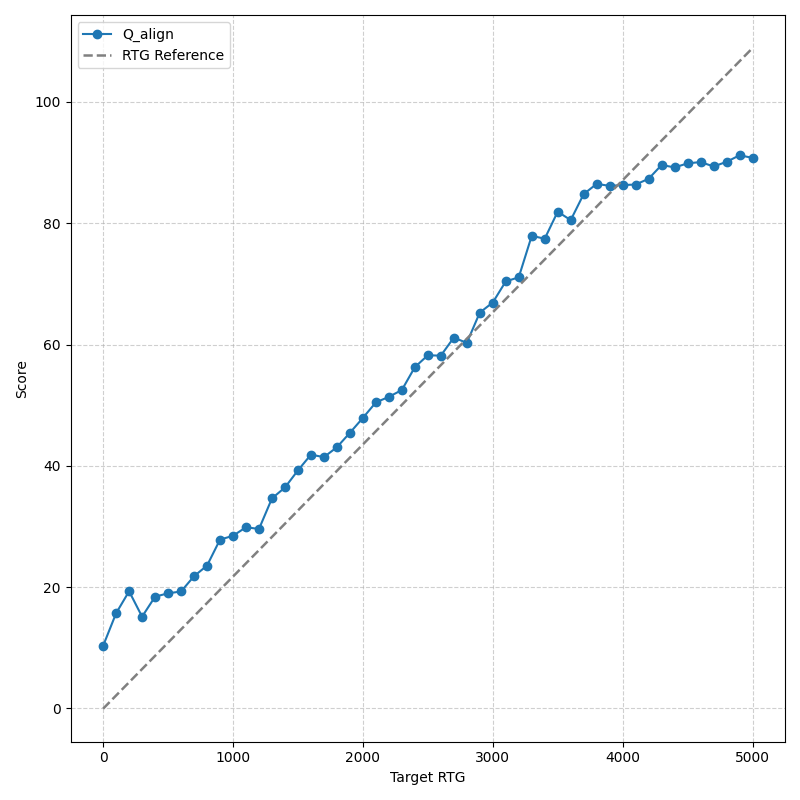}
        \caption{Walker2d Medium}
    \end{subfigure}

    \begin{subfigure}[b]{0.32\textwidth}
        \includegraphics[width=\textwidth]{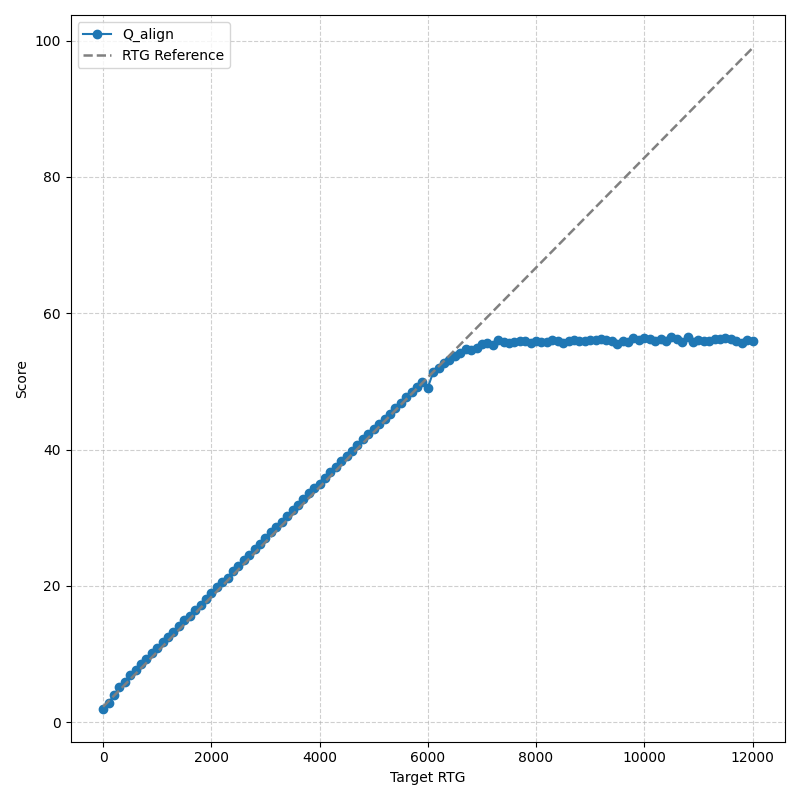}
        \caption{HalfCheetah Medium-Replay}
    \end{subfigure}
    \begin{subfigure}[b]{0.32\textwidth}
        \includegraphics[width=\textwidth]{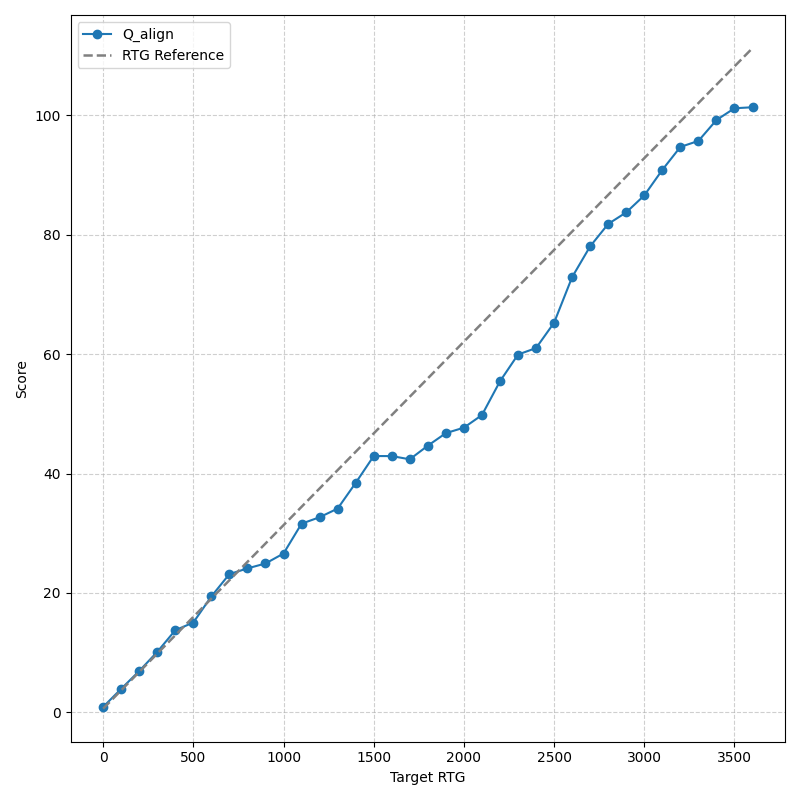}
        \caption{Hopper Medium-Replay}
    \end{subfigure}
    \begin{subfigure}[b]{0.32\textwidth}
        \includegraphics[width=\textwidth]{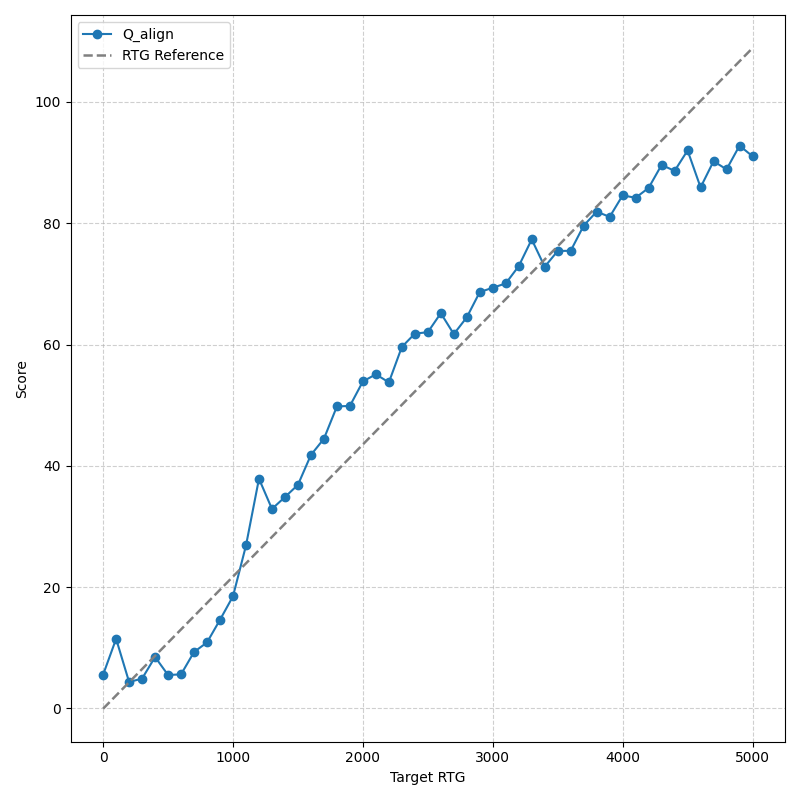}
        \caption{Walker2d Medium-Replay}
    \end{subfigure}

    \begin{subfigure}[b]{0.32\textwidth}
        \includegraphics[width=\textwidth]{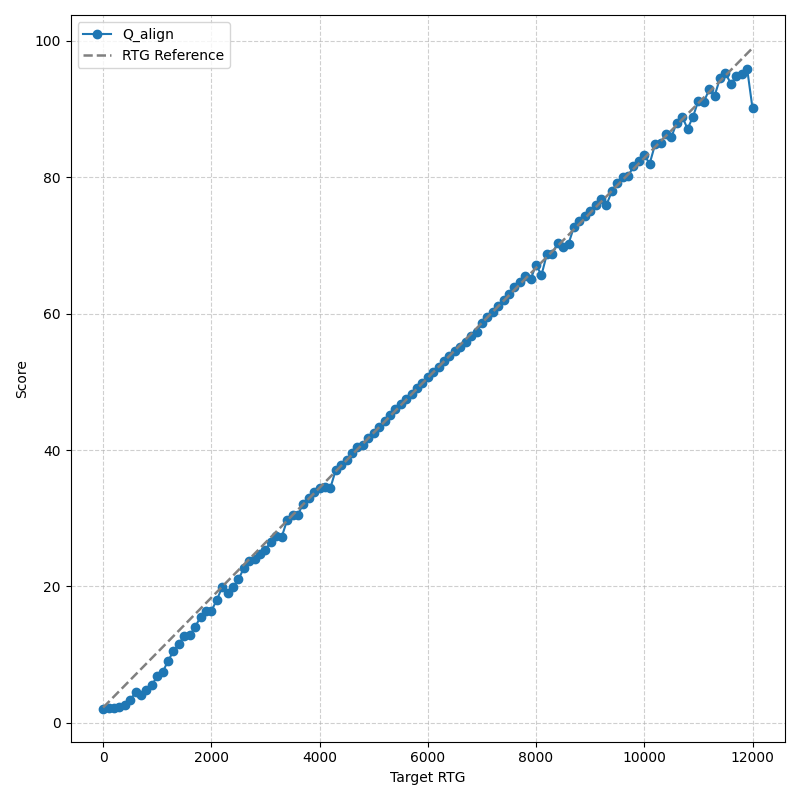}
        \caption{HalfCheetah Medium-Expert}
    \end{subfigure}
    \begin{subfigure}[b]{0.32\textwidth}
        \includegraphics[width=\textwidth]{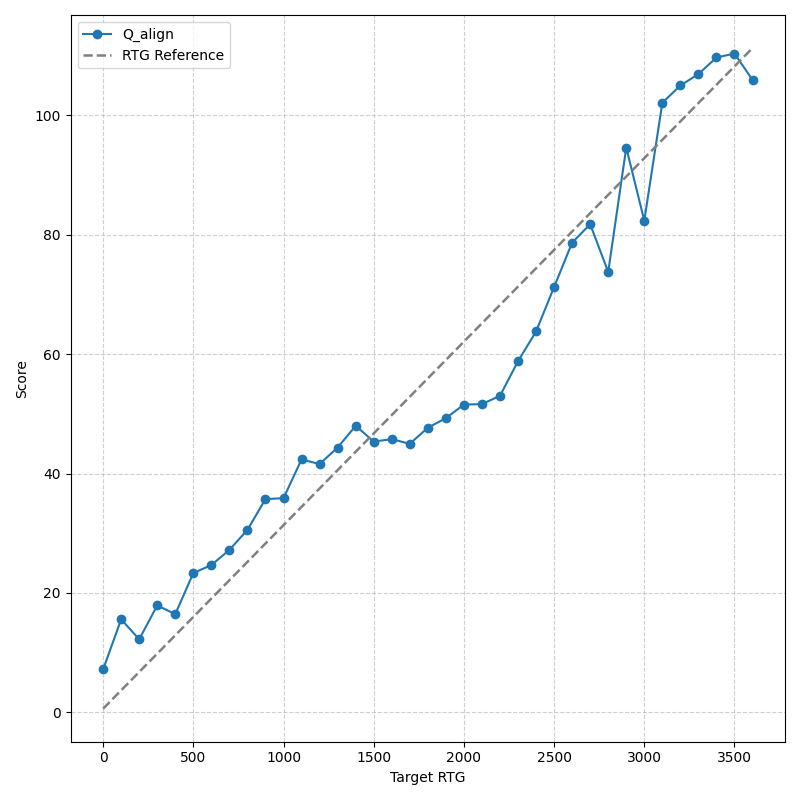}
        \caption{Hopper Medium-Expert}
    \end{subfigure}
    \begin{subfigure}[b]{0.32\textwidth}
        \includegraphics[width=\textwidth]{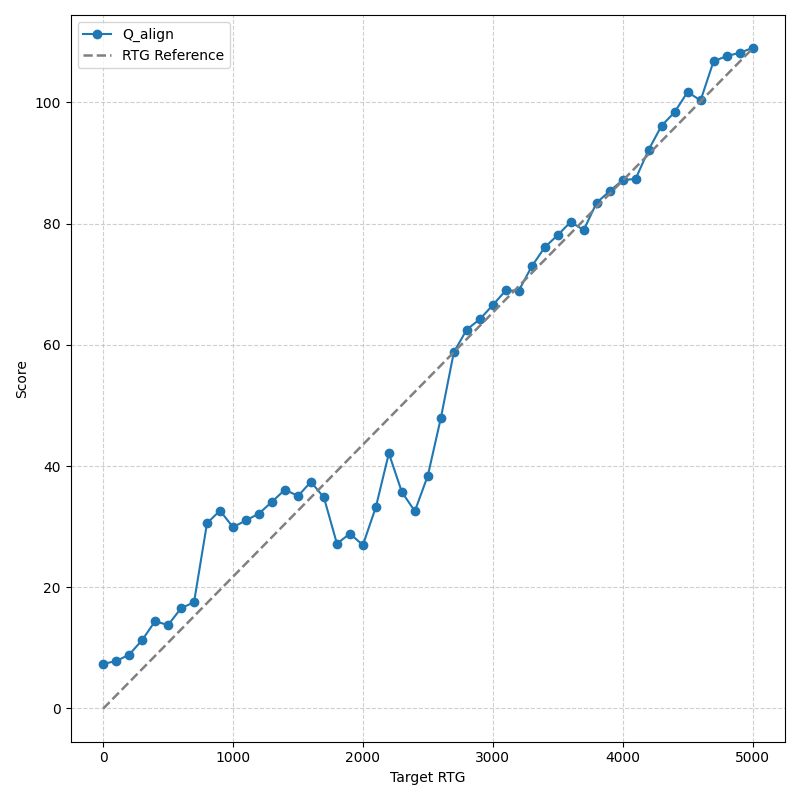}
        \caption{Walker2d Medium-Expert}
    \end{subfigure}

    \caption{Performance of \methodname~across Gym locomotion tasks. Each row corresponds to one dataset type, and each column corresponds to one environment.}
    \label{fig:gym_tasks}
\end{figure}

\begin{figure}[ht]
    \centering
    \begin{subfigure}[b]{0.4\textwidth}
        \includegraphics[width=\textwidth]{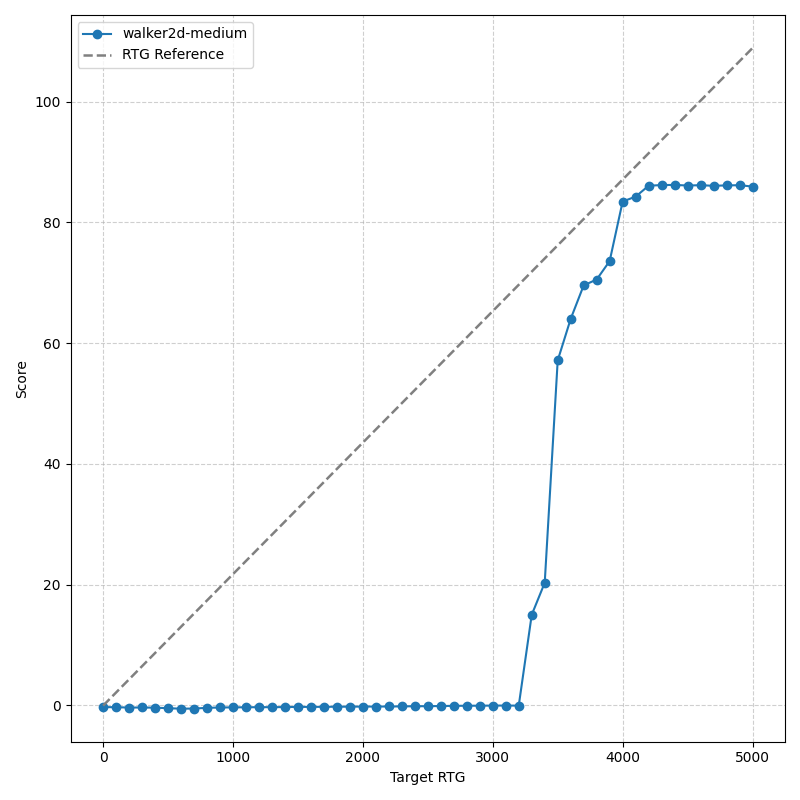}
        \caption{Walker2d Medium}
    \end{subfigure}
    \begin{subfigure}[b]{0.4\textwidth}
        \includegraphics[width=\textwidth]{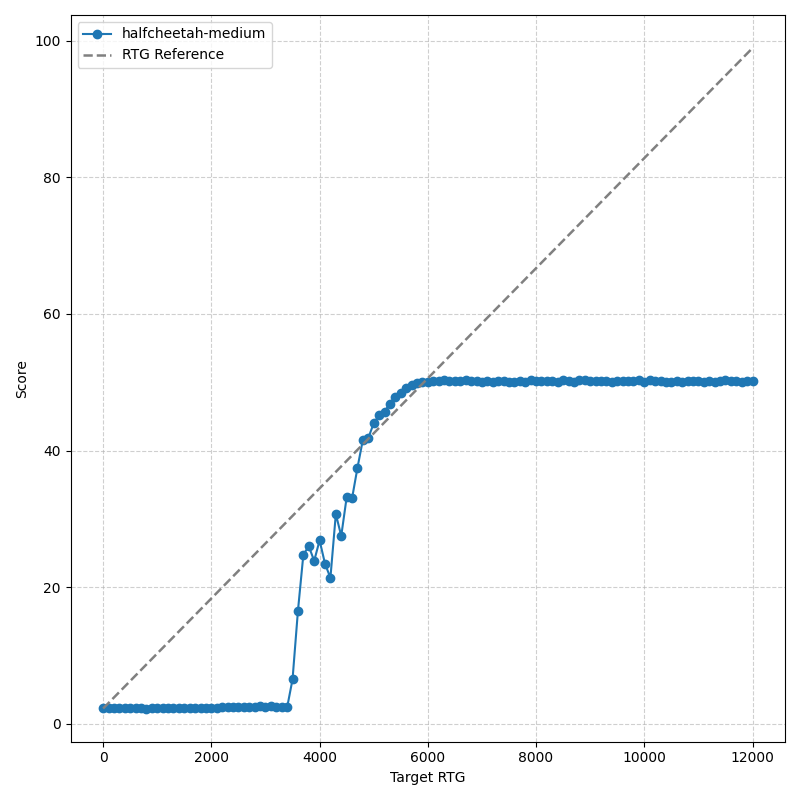}
        \caption{Halfcheetah Medium}
    \end{subfigure}
   
    \caption{Results of \texttt{walker2d-medium} and \texttt{halfcheetah-medium} when $Q$ is fixed.}
    \label{fig:No_update_Q}
\end{figure}
\begin{figure}[ht]
    \centering
    \begin{subfigure}[b]{0.4\textwidth}
        \includegraphics[width=\textwidth]{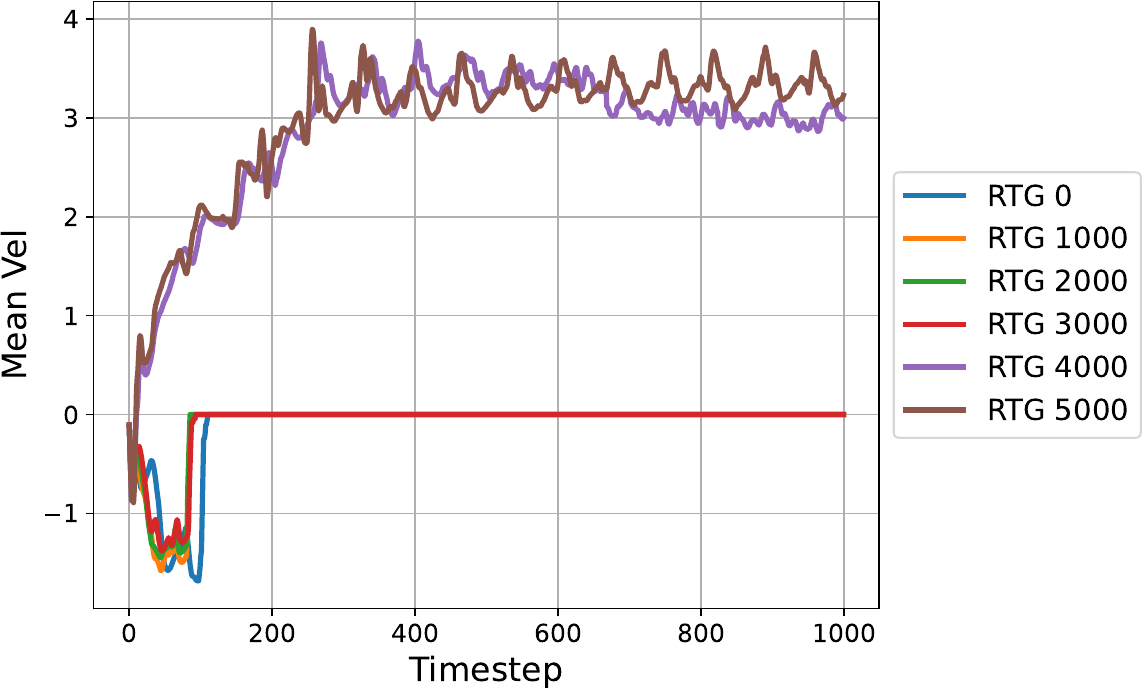}
        \caption{Walker2d Medium}
    \end{subfigure}
    \begin{subfigure}[b]{0.4\textwidth}
        \includegraphics[width=\textwidth]{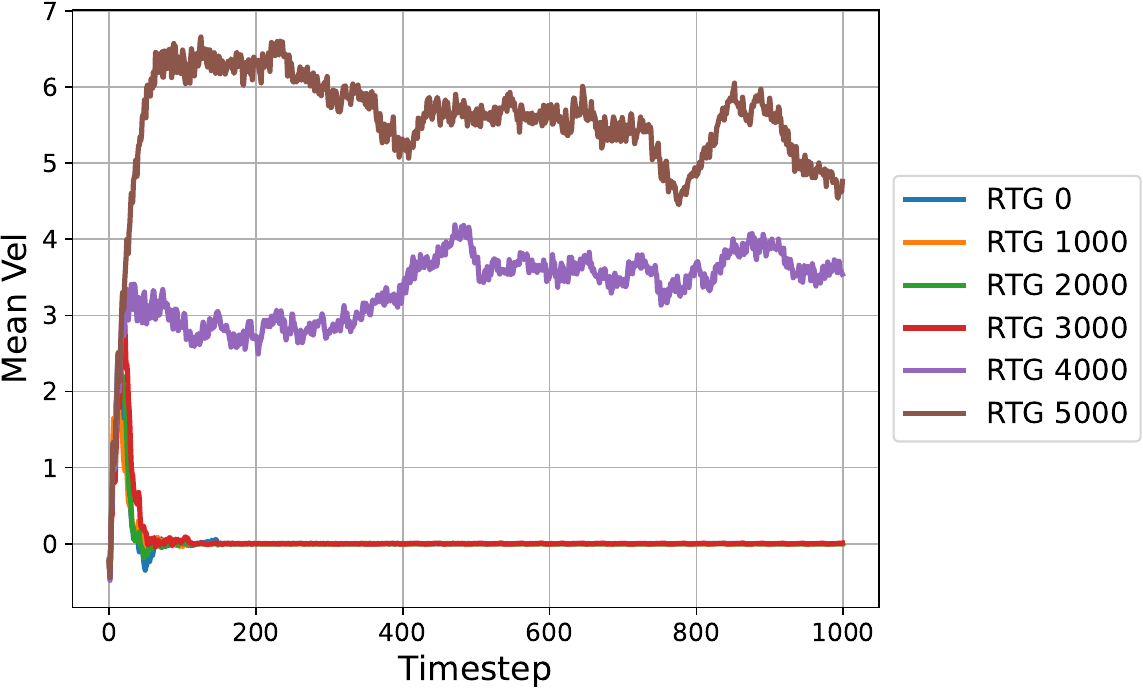}
        \caption{Halfcheetah Medium}
    \end{subfigure}
   
    \caption{Velocity on \texttt{walker2d-medium} and \texttt{halfcheetah-medium} with a fixed $Q$. At lower RTG targets, the model consistently \textbf{collapses} within the first few timesteps, rendering it \textbf{unresponsive} to RTG changes. A meaningful behavioral response to the conditioning only emerges after the RTG exceeds a critical threshold, allowing the model to escape early collapse.}
    \label{fig:Vel_No_update_Q}
\end{figure}
\begin{figure}[ht]
    \centering
    \begin{subfigure}[b]{0.32\textwidth}
        \includegraphics[width=\textwidth]{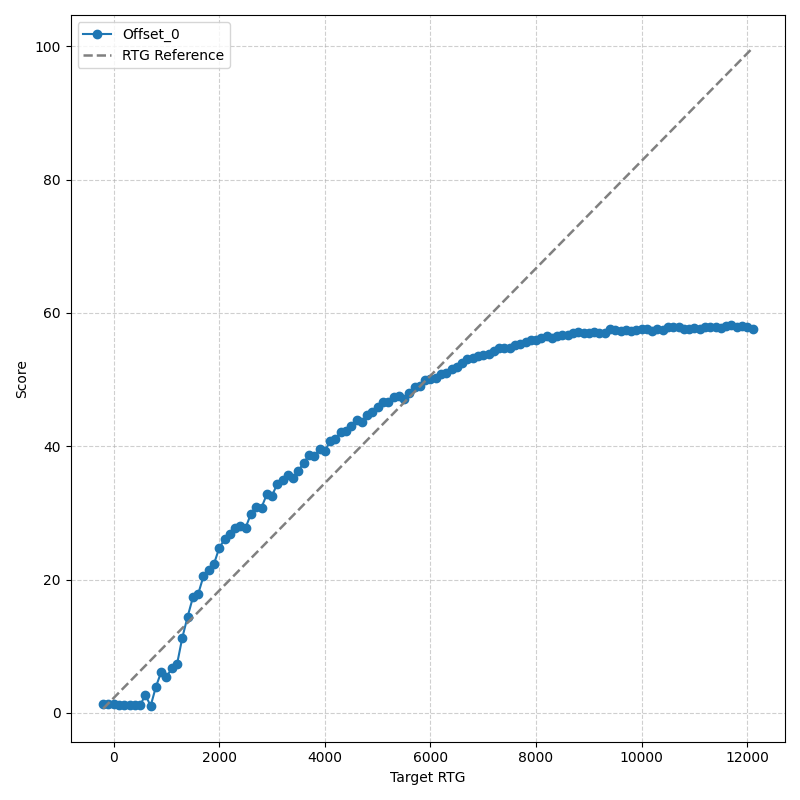}
        \caption{Offset=0}
    \end{subfigure}
    \begin{subfigure}[b]{0.32\textwidth}
        \includegraphics[width=\textwidth]{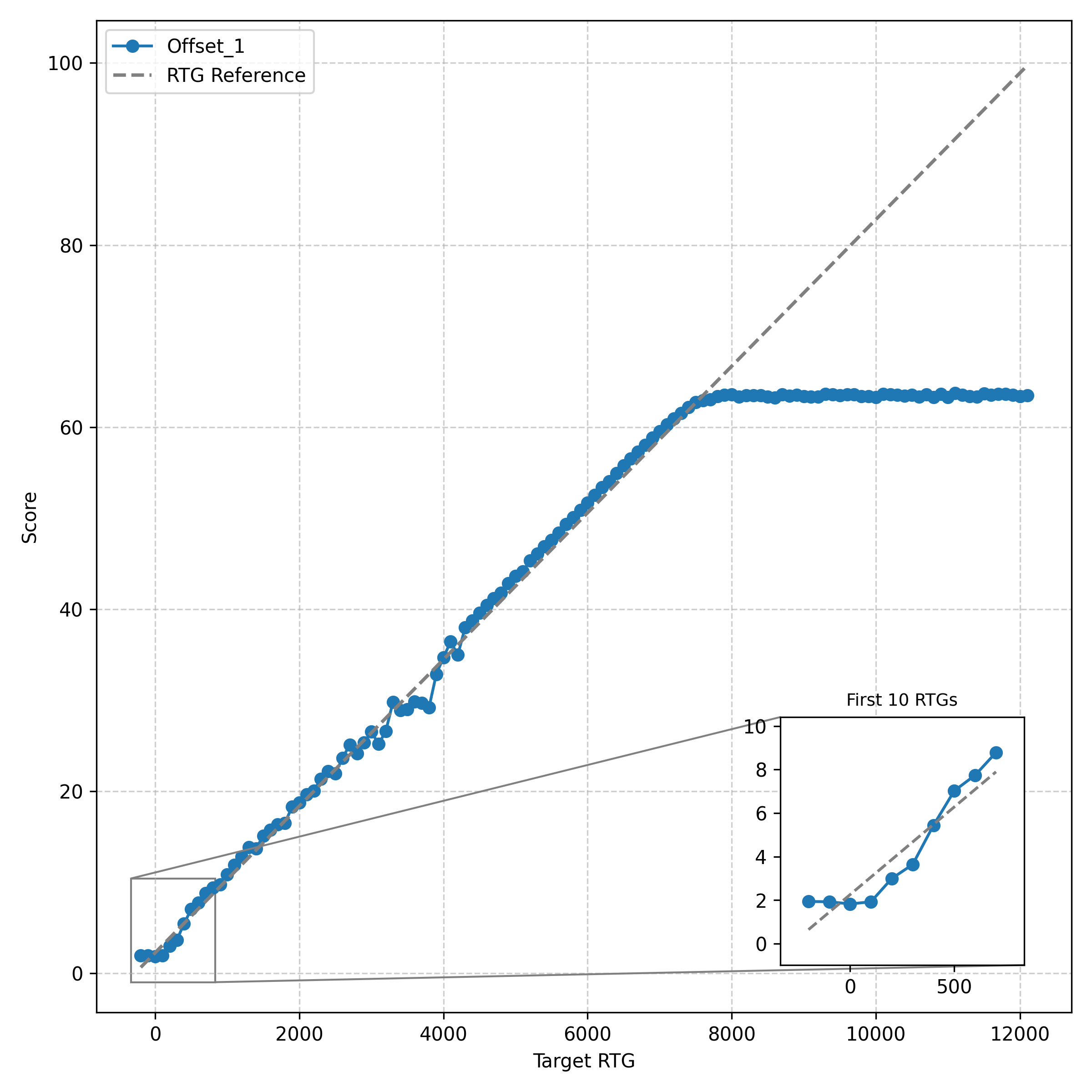}
        \caption{Offset=1}
    \end{subfigure}
    \begin{subfigure}[b]{0.32\textwidth}
        \includegraphics[width=\textwidth]{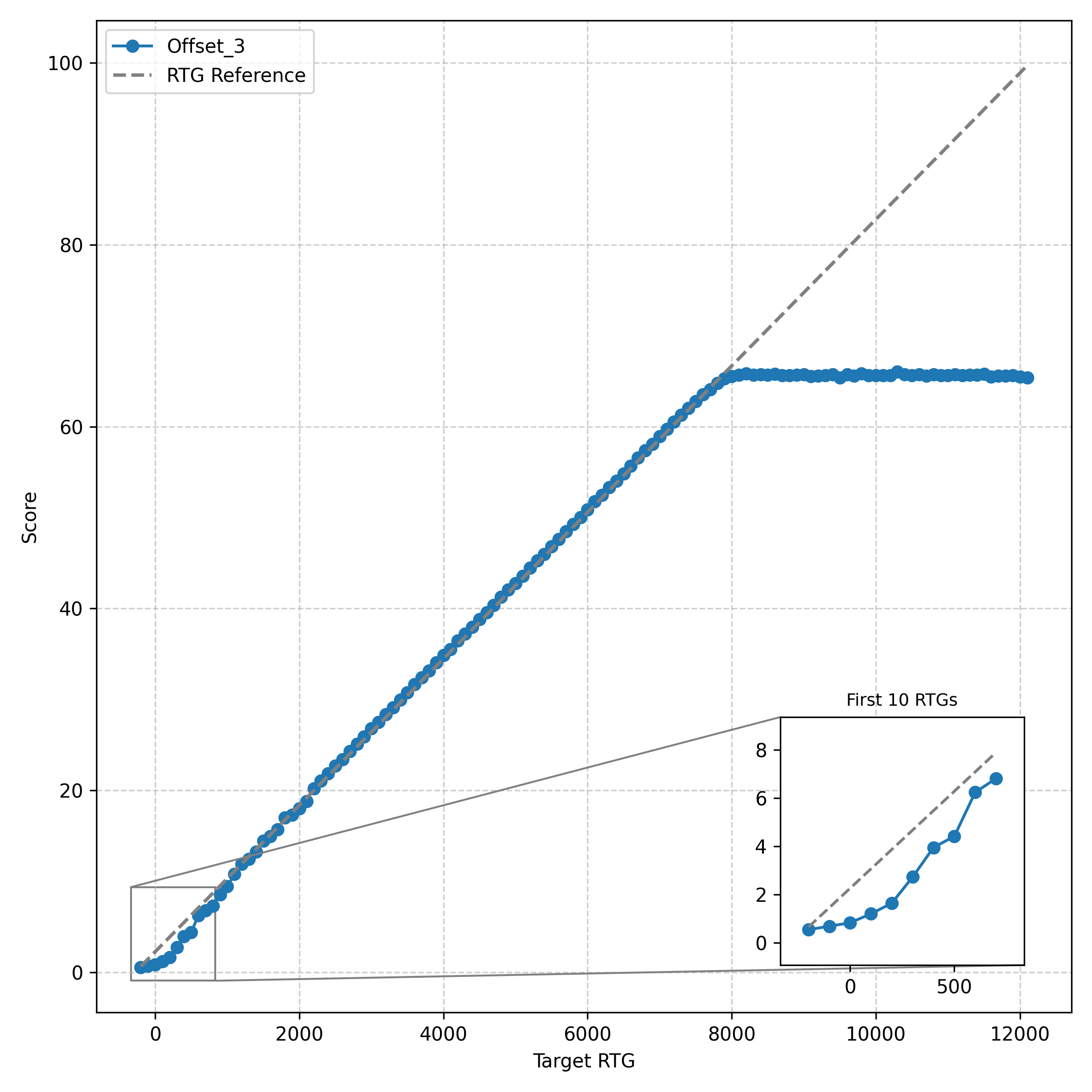}
        \caption{Offset=3}
    \end{subfigure}
    \caption{Results \texttt{halfcheetah-medium}'s alignment curve under different RTG offset.}
        \label{fig:Diff_offset}

\end{figure}

\end{document}